\documentclass[10pt,journal,compsoc]{IEEEtran}
 
\ifCLASSOPTIONcompsoc
  \usepackage[nocompress]{cite}
\else
  \usepackage{cite}
\fi

\usepackage{hyperref}       
\hypersetup{hidelinks} 

\usepackage{url}            
\usepackage{booktabs}       
\usepackage{amsfonts}       
\usepackage{nicefrac}       
\usepackage{microtype}      
\usepackage{lipsum}       
\usepackage{graphicx}
\usepackage{doi}
\usepackage{amssymb}
\usepackage{amsmath}
\usepackage{amsthm}
\usepackage[lined,linesnumbered,ruled,norelsize]{algorithm2e}
\usepackage{algorithmicx}
\usepackage{color}
\usepackage{enumitem}
\usepackage{subfigure}

\usepackage{xcolor}
\usepackage{todonotes}
\setuptodonotes{inline}

\newtheorem{lemma}{Lemma}
\newtheorem{remark}{Remark}
\newtheorem{theorem}{Theorem}
\newtheorem{proposition}{Proposition}
\newtheorem{corollary}{Corollary}

\begin{document}

\title{Collaborative Learning in General Graphs with Limited Memorization: Complexity, Learnability, and Reliability}


\author{Feng~Li,
        Xuyang~Yuan,
        Lina~Wang,
        Huan~Yang,
        Dongxiao~Yu,
        Weifeng~Lyu,
        Xiuzhen~Cheng
\IEEEcompsocitemizethanks{\IEEEcompsocthanksitem F. Li, X. Yuan, L. Wang, D. Yu and X. Cheng are with School of Computer Science and Technology, Shandong University, Qingdao, China. E-mail: \{fli, dxyu, xzcheng\}@sdu.edu.cn, \{xyyuan, linawang425\}@mail.sdu.edu.cn \protect\\
%
\IEEEcompsocthanksitem H. Yang is with College of Computer Science and Technology, Qingdao University, Qingdao, China. E-mail: cathy\_huanyang@hotmail.com \protect\\
\IEEEcompsocthanksitem W. Lyu is with School of Computer Science and Engineering, Beihang University, Beijing, China. E-mail: lwf@nlsde.buaa.edu.cn \protect\\
}
%
}

%

\IEEEtitleabstractindextext{%
\begin{abstract}
  We consider a $K$-armed bandit problem in general graphs where agents are arbitrarily connected and each of them has limited memorizing capabilities and communication bandwidth. The goal is to let each of the agents eventually learn the best arm. Although recent studies show the power of collaboration among the agents in improving the efficacy of learning, it is assumed in these studies that the communication graph should be complete or well-structured, whereas such an assumption is not always valid in practice. Furthermore, limited memorization and communication bandwidth also restrict the collaborations of the agents, since the agents memorize and communicate very few experiences. Additionally, an agent may be corrupted to share falsified experiences to its peers, while the resource limit in terms of memorization and communication may considerably restrict the reliability of the learning process. To address the above issues, we propose a three-staged collaborative learning algorithm. In each step, the agents share their latest experiences with each other through light-weight random walks in a general communication graph, and then make decisions on which arms to pull according to the recommendations received from their peers. The agents finally update their adoptions (i.e., preferences to the arms) based on the reward obtained by pulling the arms. Our theoretical analysis shows that, when there are a sufficient number of agents participating in the collaborative learning process, all the agents eventually learn the best arm with high probability, even with limited memorizing capabilities and light-weight communications. We also reveal in our theoretical analysis the upper bound on the number of corrupted agents our algorithm can tolerate. The efficacy of our proposed three-staged collaborative learning algorithm is finally verified by extensive experiments on both synthetic and real datasets.
\end{abstract}

}

\maketitle

\IEEEdisplaynontitleabstractindextext
\IEEEpeerreviewmaketitle

\section{Introduction} \label{sec:intro}
  Making a sequence of decisions to choose among a set of unknown options is a commonly encountered issue in a wide spectrum of applications, e.g., economy~\cite{ShenWJZ-ICJAI15}, robotics~\cite{PiniBFDB-ICSI12}, and biology~\cite{SeeleyB-BES99}. The problem is usually formulated as a stochastic \textit{Multi-Armed Bandit} (MAB) problem~\cite{AuerCF-ML02,BubeckB-FTML12}. Specifically, given $K$ unknown arms $a_1, a_2, \cdots, a_K$, a player (a.k.a. \textit{agent}) can select one of them to pull, and observe the corresponding reward feedback in each step. Let $\phi_{k}(r)$ denote the reward obtained by the $r$-th pull of arm $a_k$, and $\phi_{k}(1), \phi_{k}(2), \cdots$ are assumed to be i.i.d. random variables (e.g., Bernoulli variables parameterized by unknown $p_k$). The aim is to design a policy, according to which, the agents can make selection decisions sequentially to learn the best arm.
  
  Recent studies, e.g., \cite{NayyarKJ-TCNS18,WangPAJR-AISTATS20,BubeckLPS-CLT20}, have investigated a variation of the MAB problem where multiple agents independently make decisions on which arm to pull; nevertheless, most of them focus on addressing the collisions/interference among the agents, whereas a very handful of recent proposals utilize the power of the collaboration among the agents. Inspired by the fact that individuals in a social group (such as human society, social insect colonies and swarm robotics) can learn the experience from their peers~\cite{SeeleyB-BES99,BianconiB-PRL01,PrattSMF-AB05,BehimTE-EHB14,GranovskiyGSG-TCS15}, a collaborative learning dynamics consists of the following two stages in each step: in the \textit{sampling} stage, each agent chooses one of the arms to pull based on the suggestions received from its peers, while in the \textit{adopting} stage, the agent decides whether or not to adopt the chosen arm as preference according to its random reward feedback. In fact, the above two-staged collaborative learning dynamics has been investigated in recent studies \cite{CelisKV-PODC17,SuZL-MACS19}. Unfortunately, the existing proposals consider either complete graphs or well-structured ones such that the information exchange among the agents can be guaranteed. For example, in \cite{CelisKV-PODC17}, the agents can directly communicate with each other and observe the exact popularity of each arm. \cite{SuZL-MACS19} assumes the communication graph is regular or doubly-stochastic such that asynchronous communications among the agents are sufficient to serve the learning goal. Therefore, it is very challenging to enable efficient collaboration among the agents for learning in a general graph with arbitrary topology.
  
  The collaboration among the agents is also restricted by their limited memorizing capacities. Although there have been many asymptotically optimal algorithms and efficient finite-time order optimal algorithms proposed in recent decades \cite{AuerCF-ML02,BubeckB-FTML12,ShahrampourNT-TSP17} to address the MAB problem with single agent, these state-of-the-art algorithms usually have highly non-trivial requirements on the memorizing capability of the agent. For example, in the well-known \textit{Upper Confidence Bound} (UCB) algorithm, the agent is required to memorize both the cumulative reward obtained by pulling each arm and the number of pulls of each arm so far~\cite{AuerCF-ML02}. The UCB algorithm also has been applied in collaborative learning where agents share their experiences with each other and each of them makes decisions according to the historical information received so far \cite{KollaJG-TON18,SankGS-PMACS19}. In another word, it entails full historical information so far to make decision currently, and any rule with this property requires agents to have good memorizing capabilities. Unfortunately, such a requirement may not always be fulfilled. In particular, agents may have limited memory such that only quite few historical experiences can be memorized \cite{Robbins-PNAS56,CoverH-TIT70,CelisKV-PODC17,SuZL-MACS19}. For example, in human society, when a customer is making a purchase decision of perishable products, he may memorize only his most recent purchase. Similarly, in ant colonies, an ant can memorize only few recently visited sites when hunting house. The question is, when the individual agents have limited memorizing capabilities, are they able to learn the best arm through collaborating with each other?

  Our another concern is the fault tolerance of the collaborative learning process. Some agents may be corrupted by an adversary to elaborately share falsified experiences with their peers. In the adversarial setting, complete historical information may be very helpful for agents to make right decisions \cite{VialSS-MobiHoc21}; nevertheless, when the agents have bounded memorizing capabilities such that only few historical experiences can be memorized and shared, whether or not the limited collaboration can be exploited to reliably serve the learning goal is still an open problem.

  In this paper, we propose a collaborative learning algorithm for multi-agent MAB problem in a general graph where each agent has bounded memorizing capacity and may be corrupted to disseminate (or share) falsified experiences. Specifically, our algorithm proceeds iteratively and each round of our algorithm includes the following three stages:
  \begin{itemize}
    \item \textbf{Disseminating}: For each agent, if it has a preference over the $K$ arms (and thus has a \textit{non-null} adoption), it disseminates its adoption over the graph through \textit{Metropolis-Hasting Random Walks} (MHRWs) in parallel.
    \item \textbf{Sampling}: For each agent with no preference (and thus with a \textit{null} adoption), with probability $\mu$, it chooses one of the $K$ arms uniformly at random to pull; with probability $1-\mu$, it uniformly chooses one of the arms suggested by its peers in the last disseminating stage or chooses no arm if there is no suggestion received. For the agents with non-null adoptions, each of them makes its sample decision by following the second branch (i.e., by letting $\mu=0$).
    \item \textbf{Adopting}: If pulling the arm yields reward, the agent updates its adoption (or preference) over the arms; otherwise, it keeps its adoption unchanged.
  \end{itemize}

  We study the dynamics of the above three-staged collaborative learning algorithm from the perspectives of complexity, learnability and reliability, respectively. We demonstrate that, our MHRW-based information disseminating mechanism entails only light-weight communications over the general communication graph. Assume there are $N$ agents participating in the collaborative learning process. In each round, every agent only needs to transmit $\mathcal O(\log^3 N)$ messages to each of its neighbors, while each message consists of $\mathcal O(\log N)$ bits. Thanks to the above information disseminating mechanism through which the agents share their experiences efficiently, we demonstrate that, even the agents have limited memorizing capabilities and thus memorize only the latest adoptions, the learnability of our algorithm can be guaranteed with high probability when $N$ is sufficiently large. Furthermore, we quantify the reliability of collaborative learning algorithm. Specifically, let $p_1$ and $p_2$ denote probabilities for the best arm and the second best one to yield reward, respectively. Up to $\frac{(1-\alpha)(p_1-p_2)}{(1-\alpha)p_1+\alpha p_2}N$ corrupted agents can be tolerated in each round $r$ when the proportion of honest agents adopting the best arm in round $r$ is at most $0<\alpha<1$.

  The remaining of this paper is organized as follows. We first survey related literature in Sec.~\ref{sec:relwork}. We then introduce our system model and formulate our problem in Sec.~\ref{sec:sys}. The details of our three-staged collaborative learning algorithm and the corresponding theoretic analysis are then given in Sec.~\ref{sec:algo} and Sec.~\ref{sec:analysis}, respectively. We also perform extensive numerical experiments to verify the efficacy of our proposed algorithm in Sec.~\ref{sec:sim}. We finally conclude this paper in Sec.~\ref{sec:con}.

\vspace{-2ex}
\section{Related Work} \label{sec:relwork}
  MAB is a very powerful framework for designing algorithms which make decisions over time with uncertainty \cite{Slivkins-FTML19}. Although there have been many proposals investigating the single agent MAB problem (e.g., \cite{LaiR-AAM85,AuerCF-ML02,AudibertBM-COLT10,BubeckB-FTML12,KaufmannCG-JMLR16}), studies on how multiple agents learn collaboratively were rather rare until recent years. 

  The power of the collaboration to improve the efficiency of learning process has been revealed in \cite{TaoZZ-FOCS19}; nevertheless, most existing proposals focus on utilizing rich historical information (e.g., the cumulative reward obtained by pulling each arm and the total number of the pulls of each arm so far) in a distributed manner. Therefore, those methods to address the single agent MAB problem (e.g., UCB method, $\varepsilon$-greedy method and SoftMax method \cite{AuerCF-ML02,VermorelM-ECML05}) are still very useful for resolving the multi-agent MAB problem. In \cite{ChakrabortyCDJ-IJCAI17}, each agent either chooses one of the arms to pull or broadcasts its local historical experience in each round. The agents which choose to pull arms make their decisions based on SoftMax method. \cite{SzorenyiBHOJB-ICML13} proposes a gossip-based algorithm to address the MAB problem in \textit{Peer-to-Peer} (P2P) networks. Specifically, in every round, each agent first shares its empirical data to two randomly chosen neighbors and then performs $\varepsilon$-greedy method to make a decision to choose an arm to pull. However, this algorithm relies on constructing an overlay network with special topology. General social graphs (with arbitrary topology) are considered in \cite{KollaJG-TON18} where a hierarchical learning algorithm is designed. In \cite{KollaJG-TON18}, the dominating set of the graph should be first recognized. The agents in the dominating set (i.e., so-called ``leaders'') apply an UCB-based learning policy to choose among the arms according to the historical experiences collected from their one-hop neighbors, while each of the others (i.e., the ones who are not in the dominating set) makes the same decision as its leader. The UCB policy is also used in \cite{SankGS-PMACS19}; it assumes only a limited number of bits (i.e., the ID of the recommended arm) can be shared by each agent to a random peer.

  However, rich historical information may not always be available for individual agents, since an agent may not have sufficient memorizing capability. Referring to human choice behavior~\cite{BehimTE-EHB14,BianconiB-PRL01,GranovskiyGSG-TCS15} and animal behavior~\cite{PrattSMF-AB05,SeeleyB-BES99}, a two-staged algorithmic paradigm for collaborative learning is considered in \cite{CelisKV-PODC17,SuZL-MACS19}, which includes sampling stage and adopting stage as mentioned in Sec.~\ref{sec:intro}. In \cite{CelisKV-PODC17}, arms are sampled according to their popularities, calculating which in a complete graph is easy but may induce considerable communication overhead in a general communication graph. \cite{SuZL-MACS19} considers memory-bounded agents, each of which has a finite-valued memory such that only the latest adoption can be memorized~\cite{Cover-IC68,CoverH-TIT70}. The agents are asynchronous; hence, each of them randomly chooses only one of its neighbors to acquire recommendations in each round. Although the agents have limited memorizing capabilities, \cite{SuZL-MACS19} considers a well structured communication graph (e.g., a regular graph) and the agents can eventually learn the best arm through the asynchronous collaboration over the graph. In contrast, we take into account communication graphs with general topology in our algorithm. Parallel random walks are efficiently conducted over the general graph such that each agent can sufficiently share its latest experience with others through light-weight communications. By fully exploiting the synchronous collaboration among the agents, both learnability and reliability of our collaborative learning process can be ensured.

  Different from the stochastic MAB problems investigated in the above proposals, another variant of MAB is (non-stochastic) adversarial MAB where reward feedback is controlled by an adversary~\cite{BubeckB-FTML12,LykourisML-STOC18,KapoorPK-ML19,GuptaKT-COLT19}. \cite{VialSS-MobiHoc21} takes into account an adversarial setting which is similar to ours. In \cite{VialSS-MobiHoc21}, arms yield Bernoulli rewards and a malicious agent recommends an arbitrary arm instead of the one which it believes is the best. Although the collaborative learning algorithm proposed in \cite{VialSS-MobiHoc21} is of high robustness in face of malicious node, the agents adopt UCB policy to make decisions on which arm to pull with no memorization constraint considered.
  %

\section{System Model and Problem Description} \label{sec:sys}
  In this section, we first introduce our multi-agent graph model in Sec.~\ref{ssec:graph}. We then describe our collaborative learning problem in Sec.~\ref{ssec:collalearning}. We finally introduce the bounded memory model and adversary setting in Sec.~\ref{ssec:memory} and Sec.~\ref{ssec:adversary}, respectively. For ease of understanding, frequently used notations throughout this paper are summarized in Table~\ref{tab:notation}.
  \begin{table*}
  \caption{Frequently used symbols and notations.} \label{tab:notation}
  \vspace{-2ex}
	\begin{tabular}[t]{|p{3.cm}|p{14cm}|}
	\hline
	  $\mathcal G=(\mathcal N, \mathcal E)$ & A social graph consisting of agents $\mathcal N$ and communication edges $\mathcal E$ \\ \hline
    $N = |\mathcal N|$ & The number of agents \\ \hline
	  $\mathcal N_i \subseteq \mathcal N$, $d_i = |\mathcal N_i|$ & $\mathcal N_i$ denotes the set of the neighbors of agent $i$ and $d_i$ is the number of the neighbors of agent $i$\\ \hline
	  $\mathcal A = \{a_1, a_2, \cdots, a_K\}$ & A set of $K$ arms  \\ \hline
	  $\phi_k(r) \sim \mathsf{Bernoulli}(p_k)$  & The reward obtained by pulling $a_k$ in round $r$, which obeys a Bernoulli distribution parameterized by $p_k$  \\ \hline
	  $X_{i,k}(r) \in \{0,1\}$  & A variable indicating if (honest) agent $i$ adopts arm $a_k$ as a preference in round $r$  \\ \hline
	  $Z_k(r)$ & The number of (honest) agents adopting arm $a_k$ in round $r$   \\ \hline
	  $Q_k(r)$ & The popularity of arm $a_k$ among the (honest) agents in round $r$  \\ \hline
	  $\mathcal V_i(r)$ & The set of tokens (or arm recommendations) received by agent $i$ in round $r$  \\ \hline
    $\mathcal V_{i,k}(r)$ & The set of $a_k$-tokens (i.e., tokens recommending $a_k$) received by agent $i$ in round $r$ \\ \hline
    $Q_{i,k}(r)$ & The proportion of $a_k$-tokens among the ones received by agent $i$ in round $r$ \\ \hline
    $a_i(r)$, $\phi_{a_i(r)}$ & $a_i(r)$ denotes the arm sampled by agent $i$ in round $r$ and  $\phi_{a_i(r)}$ is the reward obtained by pulling $a_i(r)$\\ \hline
	  %
    %
	  %
	  %
	  %
	  $\omega_i$  & A variable indicating the adoption of agent $i$ which may be updated in each round   \\ \hline
    $M(r)$ & The number of tokens disseminated in the disseminating stage of round $r$  \\ \hline
    $M_k(r)$ & The number of $a_k$-tokens disseminated in the disseminating stage of round $r$  \\ \hline
    $\tau$ & The proportion of corrupted agents  \\ \hline
    $T$ & The number of slots in each round  \\ \hline
    \end{tabular}
  \end{table*}

  \subsection{Multi-Agent Graph} \label{ssec:graph}
    We consider a communication graph $\mathcal G=(\mathcal N, \mathcal E)$, where $\mathcal N=\{1,2,\cdots,N\}$ denotes a group of agents and $\mathcal E$ is a set of edge among the agents $\mathcal N$. If there is an edge between agents $i$ and $i' \in \mathcal N$, they can exchange messages with each other. Each agent $i$ has a set of neighbors $\mathcal N_i \subseteq \mathcal N$ and let $d_i = |\mathcal N_i|$ denote the degree of agent $i$. We suppose that $\mathcal G$ is \textit{connected} and \textit{non-bipartite}. This assumption has been extensively used in designing and analyzing distributed graph algorithms~\cite{DasMP-DISC12,DasNPT-JACM13}. Note that the assumption is only for our theoretical analysis and our algorithm still works even the assumption does not strictly hold. We also suppose that the agents are synchronized such that time can be divided into a sequence of time slots $t=1,2,\cdots$. We employ the CONGEST model to characterize the communications among the agents, which has been highly recognized in the field of distributed computing and communications \cite{Peleg-book00,GonenO-OPODIS17,CensorFSV-DC19,AgarwalR-SPAA20,HalldorssonKMT-STOC21}. By the CONGEST model, an agent transmits up to $\mathcal O(\log N)$ messages to each of its neighbors in a slot, while each message consists of $\mathcal O(\log N)$ bits. As will be shown in Sec.~\ref{sec:algo}, our algorithm proceeds iteratively. Each round $r$ is composed by $T$ slots \footnote{The length of each round, i.e. $T$, will be discussed later in Sec.~\ref{ssec:complexity}.} and each agent chooses one arm to pull in each round.

  \subsection{Collaborative Learning} \label{ssec:collalearning}
    The agents collaboratively solve a $K$-armed stochastic bandit problem. We denote by $\mathcal A = \{a_1, a_2, \cdots, a_K\}$ a set of $K$ arms. For each arm $a_k \in \mathcal A$, the reward process is a Bernoulli process parameterized by $p_k$. In another word, if arm $a_k$ is pulled in round $r$, the obtained reward $\phi_{k}(r) \sim \mathsf{Bernoulli}(p_k)$ such that $\mathbb P(\phi_{k}(r)=1) = p_k$ and $\mathbb P(\phi_{k}(r)=0) = 1-p_k$. Without loss of generality, we assume there exists a unique best arm (i.e., $a_1$) and $p_1 > p_2 \geq p_3 \geq \cdots \geq p_K \geq 0$. We also suppose that $p_1, p_2, \cdots, p_K$ are unknown to the agents initially. Our goal is to design a learning algorithm, based on which, the agents can collaboratively make decisions sequentially to choose among the arms to pull, in order to learn the best arm $a_1$ according to the reward feedback.

    We assume that an agent has at most one adoption (or preference) over the $K$ arms in each round. Since an agent may have no preference over the $K$ arms, we hereby introduce a ``virtual'' arm $a_0$ (which is called \textit{null} arm in the following) such that an agent is said to virtually adopt $a_0$ if it has no preference. Let $X_{i,k}(r) \in \{0,1\}$ be a binary variable indicating if agent $i$ adopts (or prefers) arm $a_k$ ($k=0,1,\cdots,K$) in round $r$. It is apparent that $\sum^K_{k=0} X_{i,k}(r)=1$ for any agent $i$ in round $r$. Then, the adoption state of the $N$-agent system can be represented by $\mathcal{X}(r) = \{ X_{i,0}(r), X_{i,1}(r), \cdots, X_{i,k}(r) \}_{ i \in \mathcal N}$. If $X_{i,0}(r)=1$ (or $X_{i,k}(r) = 0$ for any $k=1,\cdots,K$), agent $i$ is said to have a \textit{null} adoption in round $r$. Suppose $Z_k(r) = \sum^N_{i=1} X_{i,k}(r)$ denotes the number of the agents adopting arm $k$ in round $r$ and $Q_k(r) = Z_k(r)/\sum^K_{k'=1}Z_{k'}(r)$ is thus the popularity of arm $a_k$ in round $r$. Different from most existing proposals (e.g., \cite{AuerCF-ML02,BubeckB-FTML12,ShahrampourNT-TSP17,CelisKV-PODC17,KollaJG-TON18,SankGS-PMACS19}) which study the cumulative reward obtained within finite time horizon, we follow \cite{SuZL-MACS19} to define the learnability of our algorithm. Specifically, our learning process is said to succeed when all the agents eventually learn the best arm $a_1$. We formally define the success event for the collaborative learning process as
    \begin{align} \label{eq:succevent}
      \mathsf{Succ}(N) \triangleq& \{\text{Each agent eventually learns the best arm $a_1$}\}  \nonumber\\
      =& \left\{ \lim_{r \rightarrow \infty} Z_1(r) = N \right\}
    \end{align}
    As will be shown in Sec.~\ref{sec:analysis}, the success event holds with high probability when there are a sufficient number of agents participating in our collaborative learning process.

  \subsection{Limited memorization} \label{ssec:memory}
    We assume that each agent has limited memorizing capability for learning. In this paper, we borrow the assumption of limited memorization used in \cite{CoverH-TIT70,GhaffariP-PODC16,CelisKV-PODC17,SuZL-MACS19}. Specifically, assume that the memory state of any agent $i$ can be represented by a finite-valued variable $\omega_i \in \{0, 1,2,\cdots,K\}$. Under this memory model, an agent memorizes and recommends its most recent adoption only in each round during the collaborative learning process. It is then said that each agent has a memory of size $K+1$. Consequently, as mentioned in Sec.~\ref{sec:intro}, in our algorithm, an agent decides on which arm to pull in the sampling stage based on its current adoption and the experiences collected from its peers in the disseminating stage. It finally updates its adoption based on its current adoption and the reward obtained by pulling the arm chosen in the sampling stage.

    \subsection{Adversary model} \label{ssec:adversary}
      We assume that the agents may be corrupted by an adversary \cite{VialSS-MobiHoc21}. We suppose that the adversary have limited power such that only a (fixed) set of up to $\tau N$ agents suffer the corruptions, while the other $(1-\tau) N$ agents remain honest. In each round, the corrupted agents recommend arbitrary arms to the other honest ones. We also assume the honest agents do not know who is honest and who is corrupted.

\section{Our Collaborative Learning Algorithm} \label{sec:algo}
  In this section, we first present the details of our collaborative learning algorithm in Sec.~\ref{ssec:algodetails}. We then give a short discussion in Sec.~\ref{ssec:algodiscussion}, which is helpful to our later theoretical analysis.
  \subsection{Algorithm Details} \label{ssec:algodetails}
    Our algorithm is performed iteratively, and each agent conducts the following three stages in every round: i) in the \textbf{disseminating} stage, each agent $i$ disseminates its current adoptions (a.k.a. \textit{tokens}) to its peers through parallel \textit{Metropolis-Hasting Random Walks} (MHRWs) over graph $\mathcal G$; ii) in the \textbf{sampling} stage, agent $i$ randomly chooses one arm (denoted by $a_i(r) \in \mathcal A$) to pull based on the tokens (or arm recommendations) $\mathcal V_i(r)$ received from its peers; iii) in the \textbf{adopting} stage, each agent $i$ updates its adoption $\omega_i$ according to the reward obtained by pulling arm $a_i(r)$. The pseudo-code of our algorithm is given in \textbf {Algorithm}~\ref{alg:collearning}. We suppose that each agent $i \in \mathcal N$ initially has a null adoption $\omega_i \leftarrow a_0$.
    We also assume that each agent $i$ acquires the degrees of its neighbors (i.e., $\left\{d_{i'}\right\}_{i' \in \mathcal N_i}$) at the initialization phase. Since the neighbors of any agent are fixed, the information acquisition induces only light-weight local communications between agent $i$ and its neighbors $\mathcal N_i$.
    \begin{algorithm}[htb!]
      %
      $\rhd$ \underline{\textbf{Disseminating}}: \\
      $\mathcal V_i(r) = \emptyset$; \\
      \If{$\omega_i \neq a_0$}{
        Generate $h\log N$ tokens encapsulating $\omega_i$ and initialize the counters of the tokens to $T$; \\
        Push the tokens into the FIFO queue;
      }
      \ForEach{$t=1,2,\cdots,\mathcal O(\log^2 N)$}{
        $\rhd$ \textit{Transmitting}: {\\
        Pop the first up to $h \log N$ tokens out of the FIFO queue; \\
        Decrease the counters of the tokens by $1$; \\
        Transmit each of the tokens to $\mathcal N_i$ according to the probability distribution (\ref{eq:transprob});\\
        }
        $\rhd$ \textit{Receiving}: {\\
        \eIf{The received token is feasible}
        {
          Push it into the FIFO queue;
        }
        {
          Merge it into $\mathcal V_i(r)$;
        }
        }
      }
      $\rhd$ \underline{\textbf{Sampling}}: \\
      \eIf{$\omega_i=a_0$}{
        With probability $\mu$, set $a_i(r)$ to be one of the $K$ arms uniformly at random; \\
        With probability $1-\mu$:\\
        \begin{enumerate}
          \item[a.] If $\mathcal V_i(r) = \emptyset$, $a_i(r) \leftarrow a_0$; \\
          \item[b.] If $\mathcal V_i(r) \neq \emptyset$, uniformly choose $a_i(r) \in \mathcal V_i(r)$;
        \end{enumerate}}
      {
        If $\mathcal V_i(r) = \emptyset$, $a_i(r) \leftarrow a_0$; \\
        If $\mathcal V_i(r) \neq \emptyset$, uniformly choose $a_i(r) \in \mathcal V_i(r)$;
      }
      $\rhd$ \underline{\textbf{Adopting}}: \\
      \eIf{$a_i(r) = 0$}{
        Pull no arm and let $\phi_{a_i(r)}=0$;
      }{
        Pull arm $a_i(r)$ and observe reward $\phi_{a_i(r)}$;
      }
      \eIf{$\phi_{a_i(r)} = 0$}{
        $\omega_i$ unchanged;
      }{
        $\omega_i \leftarrow a_i(r)$;
      }
    \caption{Our collaborative learning algorithm (at each agent $i$ in round $r$).} 
    \label{alg:collearning}
    \end{algorithm}

    \textbf{(1) Disseminating}. 
    In the disseminating stage, each agent $i$ with non-null adoption $\omega_i\in\{a_1, \cdots, a_K\}$ encapsulates its current adoption $\omega_i$ and a length counter into a token (see Lines 2-6). The length counter indicates how many times the token will be forwarded in the residual disseminating stage and is initialized to $T \in \mathcal O(\log^2 N)$. A token is said to be \textit{feasible} if it has a non-zero length counter. Each agent $i$ then launches $h \log N$ MHRWs in parallel (where $h$ is a constant), each of which carries a copy of its token. A \textit{First-in-First-out} (FIFO) queue is used by each agent $i$ to buffer the received feasible tokens. As shown in Lines 8-11, in each slot of round $r$, agent $i$ pops the first up to $h\log N$ feasible tokens out of its queue and then forwards the tokens individually to its neighbors according to the probability distribution $\Psi(i,i')$ 
    \begin{equation} \label{eq:transprob}
      \Psi(i,i') = \begin{cases}
        \min \left\{ \frac{1}{d_i}, ~\frac{1}{d_{i'}} \right\}, ~\text{for}~\forall i' \in \mathcal N_i \\
        1 - \sum_{j \in \mathcal N_i} \Psi(i, j), ~\text{for}~i=i'
      \end{cases}
    \end{equation}
    Before being forwarded, each of the tokens has its length counter decreased by one. For each agent $i$, once receiving a token, it pushes the token into the FIFO queue if the token is feasible (with non-zero length counter); otherwise, it merges the token into $\mathcal V_i(r)$ (see Line 12-17). Note that the tokens are disseminated in a randomized manner, and an agent may receive no token such that $\mathcal V_i(r) = \emptyset$.

    \textbf{(2) Sampling}. 
    %
    %
    We let each agent either choose (or sample) one from the $K$ arms uniformly at random or make its sampling decision according to the suggestions received from its peers. In particular, for each agent $i$ in round $r$, if it does not have any preference (i.e., $\omega_i=a_0$), then
    \begin{itemize} 
      \item With probability $\mu\in[0,1]$, the agent $i$ chooses one of the $K$ arms uniformly at random (see Line 21);
      \item With probability $1-\mu$, if there is no token received (i.e., $\mathcal V_i(r) \neq \emptyset$), agent $i$ does not choose any arm  such that $a_i(r)=a_0$ (see Line 22a); otherwise, it chooses one of the suggested arms uniformly at random (see Line 22b).
    \end{itemize}
    If agent $i$ has a non-null adoption, it directly chooses one out of the suggestions $\mathcal V_i(r)$ uniformly at random, if $\mathcal V_i(r)\neq \emptyset$; otherwise, it chooses no arm (see Lines 24-25).  

    \textbf{(3) Adopting}. Each agent $i$ pulls arm $a_i(r)$ and observes the resulting reward $\phi_{a_i(r)}$. If $\phi_{a_i(r)} \neq 0$, agent $i$ updates its state $\omega_i \leftarrow a_i(r)$; otherwise, $\omega_i$ is unchanged.

  \subsection{Discussion}  \label{ssec:algodiscussion}
    According to the above algorithm, we have the following propositions which will be very helpful in our later analysis.
    \begin{proposition} \label{prop:leave}
      For every agent $i$, if there exists round $r$ such that $X_{i,0}(r) = 0$, we then have $X_{i,0}(r') = 0$ for any $r' \geq r+1$.
    \end{proposition}
    \begin{proposition} \label{prop:absorb}
      If there exists round $r$ in which all agents adopt the best arms $a_1$ such that $Z_1(r)=N$, then $Z_1(r')=N$ holds for any $r' \geq r+1$.
    \end{proposition}
    \begin{proposition} \label{prop:local}
      If there exists round $r$ in which each agent has a non-null adoption and no one adopts the best arm such that $\sum^K_{k=2} Z_k(r) = N$ and $Z_1(r)=0$, we then have $Z_1(r') = 0$ for any $r' \geq r+1$.
    \end{proposition}
    It is stated in \textbf{Proposition}~\ref{prop:leave} that an agent who already has a non-null adoption will not adopt the null arm thereafter, and \textbf{Proposition}~\ref{prop:absorb} indicates that when all the agents adopt the best arm $a_1$, they will not adopt any other arms thereafter. In \textbf{Proposition}~\ref{prop:local}, it is shown that our algorithm may be ``trapped'' such that no agent adopts the best arm. Fortunately, as will be revealed in \textbf{Theorem}~\ref{thm:learnability}, when all agents have non-null adoptions and even a small fraction of them adopt the best arm, the probability for our algorithm to be ``trapped'' is highly small, if there exist a sufficient number of agents participating in the collaborative learning process. We hereby omit the proofs as these propositions can be derived straightforwardly from the algorithm.

    It is worth noting that, in every round $r$ (especially the later one with all agents having non-null adoptions, as will be shown in \textbf{Lemma}~\ref{le:z0decrease}.), each agent $i$ may receive multiple tokens (denoted by $\mathcal V_i(r)$) from its peers in the disseminating stage, and it then chooses one from $\mathcal V_i(r)$ uniformly at random to pull in the following sampling stage. In another word, the agents make their sampling decisions according to the popularities of the arms in the received tokens. Specifically, agent $i$ first calculates the proportion of arm $a_k$ in $\mathcal V_i(r)$ as
    \begin{equation} \label{eq:truepop}
      Q_{i,k}(r) = {V_{i,k}(r)}\big/{V_i(r)}
    \end{equation}
    where $V_{i}(r) = | \mathcal V_i(r) |$ and $V_{i,k}(r) = | \mathcal V_{i,k}(r) |$ denote the number of tokens and the number of $a_k$-tokens (i.e., the tokens recommending $a_k$) received by agent $i$ in round $r$, respectively. It then independently chooses any arm $a_k \in \mathcal{V}_i(r)$ with probability $Q_{i,k}(r)$. In fact, $Q_{i,k}(r)$ can be considered as an estimate of agent $i$ on $Q_k(r-1)$. As will be shown in \textbf{Lemma}~\ref{le:lwtail4Qik}, $Q_{i,k}(r)$ is very close to $Q_k(r-1)$ for any agent $i \in \mathcal N$ with high probability, thanks to our MHRW-based information disseminating mechanism.

\section{Analysis}  \label{sec:analysis}
  We first analyze the complexity of our algorithm mainly in terms of communications in Sec.~\ref{ssec:complexity} and then the learnability in Sec.~\ref{ssec:learn}. We finally reveal the reliability of our algorithm in face of agent corruptions in Sec.~\ref{ssec:reliability}.

  \subsection{Complexity} \label{ssec:complexity}
    As illustrate in \textbf{Algorithm}~\ref{alg:collearning}, the agents make decisions in each round by very light-weight computations, while the cost of our algorithm mainly stems from the MHRW-based information disseminating stage. 
    In the disseminating stage of our algorithm, each agent launches $h \log N$ MHRWs in parallel and the question is how many slots are necessitated for the parallel MHRWs to sufficiently disseminate the tokens over graph $\mathcal G$ in each round. Before answering this question in \textbf{Theorem}~\ref{thm:complexity}, we first introduce how a single MHRW disseminates a token over graph $\mathcal G$. Specifically, in each step (or slot) of the MHRW, the token is forwarded by current agent (e.g., agent $i$) to one of its neighbors (e.g., agent $i'$) with probability $\Psi(i,i')$ (see Eq.~(\ref{eq:transprob})). The matrix $\Psi_{\mathcal G} = [\Psi(i,i')]_{i,i' \in \mathcal N}$ is the transition matrix of the MHRW on graph $\mathcal G$. When graph $\mathcal{G}$ is connected and non-bipartite, $\Psi_{\mathcal G}$ is a symmetric doubly stochastic such that the MHRW eventually achieves a unique uniform stationary distribution when the time horizon for the token forwarding approaches infinity. In another word, the token reaches any agent with probability $\frac{1}{N}$ when $T \rightarrow \infty$. Furthermore, according to \cite{LevinP-Book17}, the MHRW on graph $\mathcal{G}$ is a reversible, irreducible Markov chain, and it achieves a nearly uniform distribution $\left[\frac{1}{N} - \kappa, \frac{1}{N} + \kappa\right]$ within at most $\frac{1}{\Gamma(\Psi_{\mathcal G})}\log \left( \frac{N}{\kappa} \right)$ slots, where $\Gamma(\Psi_{\mathcal G})$ denotes the spectral gap of the transition matrix $\Psi_{\mathcal G}$ which characterizes the connectivity of graph $\mathcal G$. In particular, when the token is forwarded at most $\frac{1}{\Gamma(\Psi_{\mathcal G})}\log \left( \frac{N}{\kappa} \right)$ times, the probability to reach any agent lies in the range $\left[\frac{1}{N} - \kappa, \frac{1}{N} + \kappa\right]$. We let $\kappa = \frac{1}{N^3}$ without sacrificing the generality and rationality of our analysis, such that the MHRW achieves a nearly uniform distribution $\left[\frac{1}{N} - \frac{1}{N^3}, \frac{1}{N} + \frac{1}{N^3}\right]$  in $\mathcal O \left( \frac{1}{\Gamma(\Psi_{\mathcal G})}\log N \right)$ slots. The bound also can be applied to expander graphs which are sparse graphs with strong connectivity properties, by considering the relationship between spectral gap and expansion ratio can be characterized by Cheeger's inequalities. Since we focus on analyzing how the number of agents impacts the performance of our collaborative learning algorithm in graphs with certain connectivity properties (i.e., with certain spectral gap or expansion ratio), we simply re-write the above bound as $\mathcal O \left(\log N \right)$.

    In our algorithm, each agent launches $h \log N$ parallel MHRWs to disseminate its current adoption. To let the agents sufficiently share their adoptions with each other, each random walk needs to proceed $\mathcal{O}(\log N)$ times to achieve a nearly uniform distribution $\left[ \frac{1}{N}-\frac{1}{N^3},  \frac{1}{N}+\frac{1}{N^3} \right]$. In the following, we concentrate on revealing how many slots each round should consist of such that all the $h N\log N$ random walks in our algorithm achieve a nearly uniform distribution under the CONGEST communication model.
    \begin{theorem} \label{thm:complexity}
      In \textbf{Algorithm}~\ref{alg:collearning} where each agent launches $h\log N$ MHRWs in parallel in each round, all the random walks achieve nearly uniform distribution $\left[ \frac{1}{N}-\frac{1}{N^3}, \frac{1}{N}+\frac{1}{N^3} \right]$ in $T \in \mathcal O(\log^2 N)$ slots with high probability.
    \end{theorem}
    \begin{proof}
      According to \textbf{Algorithm}~\ref{alg:collearning}, for each agent $i$, the expected number of tokens it receives in every slot is
      \begin{eqnarray*}
        && \sum_{i'\in \mathcal N_i} \min\left\{\frac{1}{d_i}, \frac{1}{d_{i'}}\right\} \times h\log N \\
        &=& \sum_{i'\in\mathcal N_i: \frac{1}{d_i} \geq \frac{1}{d_{i'}}} \frac{h\log N}{d_{i'}} + \sum_{i'\in\mathcal N_i: \frac{1}{d_i} \leq \frac{1}{d_{i'}}} \frac{h
        \log N}{d_{i'}}  \\
        &\leq& \sum_{i'\in\mathcal N_i: \frac{1}{d_i} \geq \frac{1}{d_{i'}}} \frac{h\log N}{d_{i}} + \sum_{i'\in\mathcal N_i: \frac{1}{d_i} \leq \frac{1}{d_{i'}}} \frac{h
        \log N}{d_{i}} \\
        &=& d_i \times \frac{h\log N}{d_i} = h \log N
      \end{eqnarray*}
      According to the Chernoff-Hoeffding bound~\cite{DubhashiP-book09}, agent $i$ receives at most $2 h \log N$ tokens in each slot with probability at least $1 - \frac{1}{N^{h/3}}$. Furthermore, considering the agents employ FIFO policy to forward the tokens, the tokens received by agent $i$ in slot $t$ will be delayed for at most $t$ additional slots. Hence, $\mathcal O(\log^2 N)$ slots are sufficient for all the tokens to be forwarded $\mathcal O(\log N)$ times, which completes the proof. 
    \end{proof}
    \begin{remark}
      According to the CONGEST communication model introduced in Sec.~\ref{sec:sys}, the above theorem implies that, in our collaborative learning process, every agent communicates $\mathcal O(\log^4 N)$ bits to each of its neighbors in each round. In addition, if taking into account the detailed structure of graph $\mathcal G$ (i.e., $\Psi_{\mathcal G}$), we can characterize the communication complexity more precisely according to \textbf{Theorem}~\ref{thm:complexity}. Specifically, following the proof of \textbf{Theorem}~\ref{thm:complexity}, we derive that each round should contain $\mathcal{O} \left( \frac{1}{\Gamma^2(\Psi_{\mathcal{G}})} \log^2 N \right)$ slots such that the tokens are disseminated sufficiently over the graph.
    \end{remark}

  \subsection{Learnability} \label{ssec:learn}
    We hereby investigate the learnability of our algorithm. In this section, we focus on a simplified case with no agent corrupted and will discuss the reliability of our algorithm later in Sec.~\ref{ssec:reliability}. In the following, we first demonstrate that, with an infinite number of agents (i.e., $N \rightarrow \infty$), the number of agents adopting the best arm is increased in each round on expectation (see \textbf{Theorem}~\ref{thm:inftyincrease}). We then improve the above results by revealing how large $N$ should be such that the number of agents adopting the best arm is increased with high probability in each round (see \textbf{Lemma}~\ref{le:finiteincrease}). We finally illustrate in \textbf{Theorem}~\ref{thm:learnability} that when there are a sufficient number of agents participating in the collaborative learning process (i.e., $N$ is sufficiently large), all the agents learn the best arm eventually with high probability. Especially, when $N\rightarrow\infty$, the probability approaches $1$.

    Before showing the number of agents adopting the best arm is increased in each round on expectation when $\mathcal N \rightarrow \infty$, we first prove in \textbf{Lemma}~\ref{le:sampling} that each agent $i$ could choose an arm $a_i(r) \in \mathcal A$ to pull in the sampling stage of each round $r$ according to its actual popularity $Q_k(r-1)$ when $N \rightarrow \infty$.
    \begin{lemma} \label{le:sampling}
      Suppose there are $M(r) = h \log N \sum^K_{k=1}Z_k(r-1)$ tokens disseminated over the graph $\mathcal G$ through the parallel MHRWs in each round $r$. Let $M_k(r) = Z_k(r-1) h\log N$ denote the number of $a_k$-tokens in round $r$ and $Q_k(r-1) = Z_k(r-1)/ \sum^K_{k'=1}Z_{k'}(r-1)  = M_k(r) / M(r)$ represent the popularity of arm $a_k$ in round $r-1$ (or at the beginning of round $r$). Given the adoption state $\mathcal X(r-1)$ in round $r-1$, we have
      \begin{equation} \label{eq:sampling}
        \lim_{N \rightarrow \infty} \mathbb P(a_i(r)=a_k \mid \mathcal X(r-1)) = Q_k(r-1), ~\forall k=1,\cdots,K
      \end{equation}
    \end{lemma}
    \begin{proof}
      Recall that the probability for each token to independently reach agent $i$ in the disseminating stage is in the range $\left[ {1}/{N}-{1}/{N^3}, {1}/{N}+{1}/{N^3} \right]$ and each agent independently chooses one arm from the received suggestions uniformly at random in the sampling stage. Hence, the probability for any token (say $a_k$-token) to be chosen by agent $i$ in the sampling stage of round $r$ can be upper-bounded by
      \begin{align*} \label{eq:sampling1}
        %
        & \sum^{M(r)}_{v=1} \frac{1}{v}  \binom{M(r)-1}{v-1}\left(\frac{1}{N}+\frac{1}{N^3}\right)^{v}\left(1-\frac{1}{N}+\frac{1}{N^3}\right)^{M(r)-v}  \nonumber\\
        &= \frac{\sum^{M(r)}_{v=1} \binom{M(r)}{v}\left(\frac{1}{N}+\frac{1}{N^3}\right)^{v}\left(1-\frac{1}{N}+\frac{1}{N^3}\right)^{M(r)-v}}{M(r)}  \nonumber\\
        &\leq \frac{1}{M(r)} \left( 1 + \frac{2}{N^3} \right)^{M(r)}
      \end{align*}
      Considering there are $M_k(r)$ $a_k$-tokens in round $r$ and $M_k(r) \leq h N \log N$, we have
      \begin{align*}
        &\mathbb P(a_i(r)=a_k \mid \mathcal X(r-1)) \\
        \leq& \frac{M_k(r)}{M} \left( 1 + \frac{2}{N^3} \right)^{M(r)} \leq Q_k(r-1) \left( 1 + \frac{2}{N^3} \right)^{Nh\log N}
      \end{align*}
      Therefore,
      \begin{eqnarray} \label{eq:sampling2}
        \lim_{N \rightarrow \infty} \mathbb P(a_i(r)=a_k \mid \mathcal X(r-1)) \leq Q_k(r-1)
      \end{eqnarray}

      Similarly, for any token in round $r$, the lower bound on the probability that it is chosen by agent $i$ in the sampling stage is calculated by
      \begin{align*}
        & \sum^{M(r)}_{v=1} \frac{1}{v} \binom{M(r)-1}{v-1} \left( \frac{1}{N} - \frac{1}{N^3} \right)^v \left( 1 - \frac{1}{N} - \frac{1}{N^3} \right)^{M(r)-v} \nonumber\\
        &= \frac{\sum^{M(r)}_{v=1} \binom{M(r)}{v}\left(\frac{1}{N}+\frac{1}{N^3}\right)^{v}\left(1-\frac{1}{N}+\frac{1}{N^3}\right)^{M(r)-v}}{M(r)}  \nonumber\\
        &= \frac{1}{M(r)} \left( \left( 1-\frac{2}{N^3} \right)^{M(r)} - \left( 1-\frac{1}{N}-\frac{1}{N^3} \right)^{M(r)} \right) \nonumber\\
        %
        %
        &\geq \frac{1}{M(r)} \left( 1 - \frac{2h}{N} - \left( 1-\frac{1}{N} \right)^{Nh \log N} \right)
      \end{align*}
      where the first inequality holds since i) $\left( 1-\frac{2}{N^3} \right)^{M(r)} \geq 1 - \frac{2h\log N}{N^2} \geq 1 - \frac{2h}{N}$ due to the Bernoulli inequality, and ii) $1-\frac{1}{N}-\frac{1}{N^3} \leq 1-\frac{1}{N}$, when $N\geq 2$. Therefore, we have
      \begin{eqnarray} \label{eq:sampling3}
        &&\lim_{N \rightarrow \infty} \mathbb P(a_i(r)=a_k \mid \mathcal X(r-1)) \nonumber\\
        &\geq& \lim_{N \rightarrow \infty} Q_k(r-1) \left( 1 - \frac{2h}{N^2} - \left( 1-\frac{1}{N} \right)^{Nh \log N} \right)  \nonumber\\
        &=& Q_k(r-1)
      \end{eqnarray}
      We complete the proof by combining (\ref{eq:sampling2}) and (\ref{eq:sampling3}).
    \end{proof}

    \begin{theorem} \label{thm:inftyincrease}
      When $N \rightarrow \infty$, in any round $r$ such that $Z_1(r-1) < N$ (or $Z_1(r-1) \leq N-1$), we have
      \begin{equation} \label{eq:increase0}
        \mathbb E[Z_1(r) - Z_1(r-1) \mid Z_1(r-1) < N] > 0
      \end{equation}
    \end{theorem}  
    \begin{proof}
      Let $Z_1^+(r)$ and $Z_1^-(r)$ be the number of agents whose adoptions are changed from $a_{\geq 2} \in \{a_2, \cdots, a_K\}$ to $a_1$ in round $r$ and the number of agents whose adoptions are changed from the best arm $a_1$ to $a_{\geq 2}$ in round $r$, respectively.  Let $q_{i,k}(r) \triangleq \mathbb P(a_i(r)=a_k \mid \mathcal X(r-1))$ be the probability that agent $i$ chooses $a_k$ in the sampling stage of round $r$ conditioned on $\mathcal X(r-1)$. At the beginning of round $r$, we can divide the agents into three subsets $\mathcal S_0(r-1) = \{i\in\mathcal N \mid \omega_i(r-1)=a_0\}$, $\mathcal S_1(r-1)= \{i\in\mathcal N \mid \omega_i(r-1)=a_1\}$ and $\mathcal S_{\geq 2}(r-1) = \{i\in\mathcal N \mid \omega_i(r-1) \in \{a_2,\cdots,a_K\}\}$. In another word, $\mathcal S_0(r-1)$, $\mathcal S_1(r-1)$ and $\mathcal S_{\geq 2}(r-1)$ denote the subsets of the agents adopting $a_0$, $a_1$ and $a_{\geq 2}$ in round $r-1$, respectively. Hence, we then have $Z_0(r-1) = |\mathcal S_0(r-1)|$ and $Z_1(r-1) = |\mathcal S_1(r-1)|$. We also let $Z_{\geq 2}(r-1) = \sum^K_{k=2} Z_k(r-1) = |\mathcal S_{\geq 2}(r-1)|$ denote the number of agents adopting any $a_{\geq 2} \in \{a_2, \cdots, a_K\}$.
      
      According to \textbf{Algorithm}~\ref{alg:collearning}, we have
      \begin{itemize}
        \item For $\forall i \in \mathcal S_1(r-1)$, it adopts the same arm (i.e., $a_1$) in round $r$ with probability $q_{i,1}(r) + q_{i,0}(r) + \sum^K_{k=2} q_{i,k}(r)(1-p_k)$, while adopting one of the others (i.e., $a_{\geq 2}$) with probability $\sum^K_{k=2} q_{i,k}(r)p_k$.
        \item For $\forall i \in \mathcal S_{\geq 2}(r-1)$, it adopts the best arm $a_1$ in round $r$ with probability $q_{i,1}(r)p_1$.
        \item For $\forall i \in \mathcal S_0(r-1)$, the probability to adopt arm $a_1$ in round $r$ is $\left( \frac{\mu}{K} + (1-\mu) q_{i,1}(r) \right) p_1$.
      \end{itemize}
      Therefore, the conditional expectations of $Z^+_1(r)$ and $Z^-_1(r)$ then can be defined as
      \begin{align*}
        & \mathbb E[Z^+_1(r) \mid \mathcal X(r-1)] \nonumber\\
        =& \hspace{-2ex} \sum_{i\in \mathcal S_0(r-1)} \left( \frac{\mu}{K} + (1-\mu) q_{i,1}(r) \right) p_1  + \hspace{-2ex}\sum_{i \in \mathcal S_{\geq 2}(r-1)} q_{i,1}(r)p_1 \\
      \end{align*}
      and
      \begin{align*}
        & \mathbb E \left[ Z^-_1(r) \mid \mathcal X(r-1) \right] = \sum_{i \in \mathcal S_1(r-1)} \sum^K_{k=2} q_{i,k}(r)p_k
      \end{align*}
      respectively. According to \textbf{Lemma}~\ref{le:sampling}, when $N \rightarrow \infty$, we can re-write $Z^+_1(r)$ and $Z^-_1(r)$ as
      \begin{align} \label{eq:expz1}
        & \mathbb E[Z^+_1(r) \mid \mathcal X(r-1)] \nonumber\\
        =& Z_0(r-1)\left( \frac{\mu}{K} + (1-\mu)Q_1(r-1) \right) p_1  \nonumber\\
         & + Z_{\geq 2}(r-1)Q_1(r-1)p_1
      \end{align}
      and
      \begin{align}\label{eq:expz2}
        \mathbb E[Z^-_1(r) \mid \mathcal X(r-1)] = Z_1(r-1)\sum^K_{k=2}Q_k(r-1)p_k 
      \end{align}
      When $Z_1(r-1) < N$, we have
      \begin{align} \label{eq:learnneq}
        &\frac{\mathbb E[Z^+_1(r) \mid Z_1(r-1) < N]}{\mathbb E[Z^+_1(r) \mid Z_1(r-1)  < N] + \mathbb E[Z^-_1(r) \mid Z_1(r-1) < N]}  \nonumber\\
        &\geq {\xi_0}/{\xi_1}
      \end{align}
      where
      \begin{align*}
        \begin{cases}
          \xi_0 = (N -\mu Z_0(r-1) -Z_1(r-1)) p_1 Q_1(r-1) \vspace{2ex}\\
          \xi_1 = (N -\mu Z_0(r-1) -Z_1(r-1)) p_1 Q_1(r-1)) \\
          ~~~~~~~+ Z_1(r - 1)\sum^K_{k=2} Q_{k}(r-1)p_k
        \end{cases}
      \end{align*}
      Furthermore, since $p_2 \geq p_3 \geq \cdots \geq p_K$ and $0 \leq \mu \leq 1$, we have
      \begin{align*} 
        & Z_1(r - 1)\sum^K_{k=2} Q_{k}(r-1)p_k \nonumber\\
        \leq& Z_1(r - 1)p_2\sum^K_{k=2} Q_k(r-1)   \nonumber\\
        =& Q_1(r - 1) p_2 (N - Z_0(r-1) -Z_1(r-1))   \nonumber\\
        \leq& Q_1(r - 1) p_2 (N -\mu Z_0(r-1) -Z_1(r-1))
      \end{align*}
      By substituting the above inequality into (\ref{eq:learnneq}), we have
      \begin{align} \label{eq:learnneq3}
        &\frac{\mathbb E[Z^+_1(r) \mid Z_1(r-1) < N]}{\mathbb E[Z^+_1(r) \mid Z_1(r-1) < N] + \mathbb E[Z^-_1(r) \mid Z_1(r-1) < N]}  \nonumber\\
        &\geq \frac{p_1}{p_1 + p_2} > \frac{1}{2}
      \end{align}
      and thus $\mathbb E[Z^+_1(r) \mid Z_1(r-1) < N] > \mathbb E[Z^-_1(r) \mid Z_1(r-1) < N]$, according to which, we finally complete the proof since $Z_1(r) = Z_1(r-1) + Z^+_1(r) - Z^-_1(r)$.
    \end{proof}
    \begin{remark} \label{re:learningrate}
      The above theorem also implies that, when $N \rightarrow \infty$
      \begin{equation}
        \frac{\mathbb{E}[Z^+_1(r) \mid \mathcal X(r-1)]}{\mathbb{E}[Z^-_1(r) \mid \mathcal X(r-1)]} \geq \frac{p_1}{p_2}
      \end{equation}
      That is, if the best arm $a_1$ is much better than the second best one, we could have much more agents adopting $a_1$ on expectation in each round $r$.
    \end{remark}

    In the following, we analyze the learnability of our algorithm with a finite number of agents. To facilitate our analysis, we assume $r$ is sufficiently large such that each agent has a non-null adoption (i.e., $Z_0(r) = 0$ and thus $M(r)= h \log N \sum^K_{i=1}Z_k(r) = Nh\log N$), since $Z_0(\cdot)$ is decreased with high probability in the early phase of our collaborative learning process, as shown in \textbf{Lemma}~\ref{le:z0decrease}.
    \begin{lemma} \label{le:z0decrease}
      Recall $\mathcal{S}_0(r-1) \subseteq \mathcal{N}$ denotes the set of agents with null adoptions and $Z_0(r-1) = |\mathcal{S}_0(r-1)|$. In any round $r$ with $Z_0(r-1) = \zeta_0 N$, we have
      \begin{align}
        Z_0(r-1)-Z_0(r) \geq \frac{(1-\delta)\zeta_0 N \mu}{K}\sum^K_{k=1} p_k, ~\forall\delta \in (0,1)
      \end{align}
      hold with probability at least $1 - \exp \left( -\frac{\zeta_0 N \mu \delta^2}{2K}\sum^K_{k=1} p_k \right)$.
    \end{lemma} 
    \begin{proof}
      Let $B_i(r)$ be a random variable indicating if agent $i$ adopts a non-null arm in round $r$. Specifically, if agent $i$ has a non-null adoption in round $r$, $B_i(r)=1$; otherwise, $B_i(r)=0$. Due to \textbf{Proposition}~\ref{prop:leave}, $Z_0(\cdot)$ actually is non-decreasing; hence, $Z_0(r-1)-Z_0(r) = \sum_{i \in \mathcal S_0(r-1)} B_i(r)$. Each agent $i \in \mathcal{S}_0(r-1)$ adopts a non-null arm in round $r$ by either uniform sampling or learning from its peers; therefore, $\mathbb{P}(B_{i}(r)=1) \geq \frac{\mu}{K}\sum^K_{k=1}p_k$. We then complete the proof by applying the Chernoff-Hoeffding bound~\cite{DubhashiP-book09}.
    \end{proof}

    As mentioned in \textbf{Proposition}~\ref{prop:absorb}, when all the agents adopt the best arm as their preference in some round, they do not change their adoption thereafter. Therefore, we concentrate on analyzing the evolution of $Z_1(\cdot)$ when $1 \leq Z_1(\cdot) \leq N-1$. It is shown by the following lemmas that, with high probability, we have $Z_1(r) - Z_1(r-1) \geq 1$ when $1 \leq Z_1(r-1) \leq N-1$ when there are a sufficient number of agents participating in our collaborative learning process.
    \begin{lemma}  \label{le:finiteincrease}
      Let $\Delta Z_1(r) = Z_1(r) - Z_1(r-1)$. Assume $Z_1(r-1) = \zeta_1 N$ where $1 \leq \zeta_1 N \leq N-1$. Let $0 < \varepsilon < \frac{1}{2} \zeta_1 \left(1 - \zeta_1\right)(p_1-p_2)$ and $h \geq 24/\varepsilon^2$. When $N$ is sufficiently large such that $\frac{\log N}{N} \leq \frac{1}{8h}$, we have
      \begin{align}
        \mathbb P( \Delta Z_1(r) \geq 1 \mid Z_1(r-1) = \zeta_1N) \geq 1 - \frac{2}{1+2N\varepsilon^2}
      \end{align}
    \end{lemma}
    \begin{proof}
      As mentioned in Sec.~\ref{ssec:algodiscussion}, each agent $i$ makes its sampling decision according to $Q_{i,k}(r)$ in each round $r$. Before diving into the proof of this lemma, we first show in \textbf{Lemma}~\ref{le:lwtail4Qik} that the difference between $Q_{i,k}(r)$ and $Q_k(r-1)$ can be bounded. Due to the space limit, we present the proof of \textbf{Lemma}~\ref{le:lwtail4Qik} in the appendix.
      \begin{lemma} \label{le:lwtail4Qik}
        In each round $r$, for any $k=1, 2, \cdots, K$, we have
        \begin{align}
          &\mathbb{P} \left(Q_{i,k}(r) \geq Q_k(r-1)-\varepsilon, \forall i\in \mathcal N \mid \mathcal X(r-1) \right)  \nonumber\\
          &\geq 1 - \frac{N^{1-h\varepsilon^2/8}}{1 - {4h N^{-1} \log{N}}} - \frac{1}{N^2}, ~~0 \leq \varepsilon \leq Q_k(r-1)
        \end{align}
        and
        \begin{align}
          &\mathbb{P} \left(Q_{i,k}(r) \leq Q_k(r-1)+\varepsilon, \forall i\in \mathcal N \mid \mathcal X(r-1) \right)  \nonumber\\
          &\geq 1 - \frac{N^{1-h\varepsilon^2/8}}{1 - {4h N^{-1} \log{N}}} - \frac{1}{N^2}, ~~0 \leq \varepsilon \leq 1-Q_k(r-1)
        \end{align}
        when $h \geq 64$ and $N$ is sufficiently large such that $N \geq 4 h \log N$.
      \end{lemma}

      For each agent $i \in \mathcal N$, let $A_{i,1}(r) \in \{0,1\}$ denote a Bernoulli random variable indicating if agent $i$ adopts $a_1$ in round $r$. According to Sec.~\ref{sec:algo}, for each agent $i \in \mathcal S_{\geq 2}(r-1)$, $\mathbb P(A_{i,1}(r)=1) = Q_{i,1}(r)p_1$, while for each agent $i \in \mathcal S_1(r-1)$, $\mathbb P(A_{i,1}(r)=1) = Q_{i,1}(r) + \sum^K_{k=2}Q_{i,k}(r)(1-p_k)$. Therefore, we have $\mathbb E[A_{i,1}(r)] = Q_{i,1}(r)p_1$ for $i \in \mathcal S_{\geq 2}(r-1)$ and $\mathbb E[A_{i,1}(r)] = Q_{i,1}(r) + \sum^K_{k=2}Q_{i,k}(r)(1-p_k)$ for $i \in \mathcal S_1(r-1)$. As $Z_1(r) = \sum^N_{i=1} A_{i,1}(r)$, by applying the Chernoff-Hoeffding bound~\cite{DubhashiP-book09}, we have
      \begin{align}
        \frac{Z_1(r)}{N} \geq& \frac{1}{N} \sum_{i \in \mathcal S_1(r-1)} \left( Q_{i,1}(r) + \sum^K_{k=2}Q_{i,k}(r)p_k \right)  \nonumber\\
        &+ \frac{1}{N} \sum_{i \in \mathcal S_{\geq2}(r-1)} Q_{i,1}(r)p_1 - \varepsilon  \nonumber\\
        \geq& \frac{p_1}{N} \sum_{i \in \mathcal S_{\geq 2}(r-1)} Q_{i,1}(r) + \frac{p_2}{N} \sum_{i \in \mathcal S_{1}(r-1)} Q_{i,1}(r)  \nonumber\\
        &+ \frac{(1-p_2)Z_1(r-1)}{N} - \varepsilon
      \end{align}
      hold with probability at least $1-\exp(-2N\varepsilon^2)$ (where $\varepsilon > 0$), where the second inequality holds since $p_2 \geq p_3 \geq \cdots \geq p_K$. Considering \textbf{Lemma}~\ref{le:lwtail4Qik}, we continue the above inequality and then have
      \begin{align}
        \frac{\Delta Z_1(r)}{N} =& \frac{Z_1(r) - Z_1(r-1)}{N}  \nonumber\\
        \geq& (1-Q_1(r-1))Q_1(r-1)(p_1-p_2)  \nonumber\\
        &- (p_1 - (p_1-p_2) Q_1(r-1))\varepsilon - \varepsilon \nonumber\\
        \geq& (1-Q_1(r-1))Q_1(r-1)(p_1-p_2) - 2\varepsilon  \nonumber\\
        =& \zeta_1 \left( 1-\zeta_1 \right)(p_1-p_2) - 2\varepsilon > 0
      \end{align}
      hold with probability at least 
      \begin{align}
        & 1 - \frac{N^{1-h\varepsilon^2/8}}{1 - {4h N^{-1} \log{N}}} - \frac{1}{N^2} - \exp(-2N\varepsilon^2)  \nonumber\\
        \geq& 1 - \frac{3}{N^2} - \frac{1}{1+2N\varepsilon^2} \geq 1 - \frac{2}{1+2N\varepsilon^2}
      \end{align}
      where the first inequality holds because i) $N$ is sufficiently large such that $\frac{\log N}{N} \leq \frac{1}{8h}$ and $h \geq \frac{24}{\varepsilon^2}$ and ii) $\exp(-2N \varepsilon^2) \leq \frac{1}{1+2N\varepsilon^2}$ when $N \geq 2$, while the second one holds since $\frac{N^2}{3} \geq 1 + \frac{1}{4}N(p_1-p_2) \geq 1 + N \zeta_1 (1-\zeta_1)(p_1-p_2) \geq 1 + 2 N \varepsilon \geq 1 + 2 N \varepsilon^2$ when $N \geq 2$ and $\varepsilon < \frac{1}{2}\zeta_1(1-\zeta_1)(p_1-p_2)$. We finally complete the proof by considering $\Delta Z_1(r)$ is a positive integer such that $\Delta Z_1(r) >0$ naturally implies $\Delta Z_1(r) \geq 1$.
    \end{proof}

    \vspace{-3ex}
    \begin{remark}
      As shown by \textbf{Theorem}~\ref{thm:inftyincrease}, when $N \rightarrow \infty$, the expected rate at which $Z_1(\cdot)$ achieves $N$ in each round $r$ is $\mathbb E[Z_1(r) - Z_1(r-1) \mid Z_1(r-1) \leq N-1] \geq 1$. When the number of agents is infinite, based on \textbf{Lemma}~\ref{le:finiteincrease}, the rate of convergence of our algorithm in any round $r$ can be characterized by \textbf{Corollary}~\ref{cor:rate}.
      \begin{corollary}  \label{cor:rate}
        In any round $r$ such that $Z_1(r-1) = \zeta_1 N \geq 1$,
        \begin{align*} \label{eq:rate}
           \mathbb E[Z_1(r) - Z_1(r-1) \mid Z_1(r-1) = \zeta_1 N] \geq 1 - \frac{2(1+\zeta_1 N)}{1+2N \varepsilon^2}
        \end{align*}
        for any $\varepsilon \in \left(0, \frac{1}{2}\zeta_1(1-\zeta_1)(p_1-p_2)\right)$, when $N$ is sufficiently large such that $\frac{\log N}{N} \leq \frac{1}{8h}$ where $h \geq 24/\varepsilon^2$.
      \end{corollary}
      \begin{proof}
        Since $Z_1(r) \in \{0, 1, \cdots, N\}$, we have
        \begin{align*}
          &\mathbb E[Z_1(r) - Z_1(r-1) \mid Z_1(r-1) = \zeta_1 N] \nonumber\\
          =& \sum^0_{\ell = -Z_1(r-1)} \ell \mathbb{P}(\Delta Z_1(r) = \ell \mid Z_1(r-1) = \zeta_1 N)  \nonumber\\
          &+ \sum^{N-Z_1(r-1)}_{\ell = 1} \ell \mathbb{P}(\Delta Z_1(r) = \ell \mid Z_1(r-1) = \zeta_1 N)  \nonumber\\
          \geq& -Z_1(r-1) \mathbb{P}(\Delta Z_1(r) \leq 0 \mid Z_1(r-1) = \zeta_1 N)   \nonumber\\
          & + \mathbb{P}(\Delta Z_1(r) \geq 1 \mid Z_1(r-1) = \zeta_1 N)  \nonumber\\
          =& (\zeta_1 N+1) \mathbb{P}(\Delta Z_1(r) \geq 1 \mid Z_1(r-1) = \zeta_1 N) - \zeta_1 N
        \end{align*}
        We finally complete proof by considering \textbf{Lemma}~\ref{le:finiteincrease}      
      \end{proof}
    \end{remark}

    Based on what we have revealed in \textbf{Lemma}~\ref{le:finiteincrease}, we prove the learnability of our algorithm in \textbf{Theorem}~\ref{thm:learnability}.
    \begin{theorem} \label{thm:learnability}
      Suppose $Z_1(0) = \zeta_1 N$ (where $1 \leq \zeta_1 N \leq N-1$) initially. Let $0 < \varepsilon < \frac{1}{2}\zeta_1(1-\zeta_1)(p_1-p_2)$ and $h \geq 24/\varepsilon^2$ be constants. When $N \geq 9$ is sufficiently large such that $\frac{\log N}{N} \leq \frac{1}{8h}$ and $N > \frac{3}{2\varepsilon^2}$, we have
      \begin{equation} \label{eq:learnability}
        \mathbb P \left( \mathsf{Succ}(N) \mid Z_1(0) = \zeta_1 N \right) \geq 1 - \left( \frac{2}{2N \varepsilon^2-1} \right)^{\zeta_1 N}
      \end{equation}
      Especially, as $N \rightarrow \infty$, $\mathbb P \left( \mathsf{Succ}(N) \mid Z_1(0) = \zeta_1 N \right) \rightarrow 1$.
    \end{theorem}
    \begin{proof}
      According to \textbf{Algorithm}~\ref{alg:collearning}, the evolution of $Z_1(\cdot)$ is a Markov process with the transition probability distribution described in \textbf{Lemma}~\ref{le:finiteincrease}. The goal is to analyze the probability of $Z_1(\cdot)$ hitting $N$. Unfortunately, it is highly non-trivial to analyze the evolution of $Z_1(\cdot)$, as we can transit from state $Z_1(r-1)$ to any other ones $Z_1(r) = 0, 1, \cdots, N$. Therefore, in the following, we build a simplified Markov process whose probability of hitting $N$ can be better understood.

      Let $\{\bar r_j\}_{j\geq 1} \triangleq \{r \geq 0: \Delta Z_1(r) \neq 0\}$ denote the sequence of rounds where $Z_1(\cdot)$ jumps. Then, by assuming $\bar{Z}_1(j) \triangleq Z_1(r_j)$, we can use $\bar{Z}_1(\cdot)$ to represent the evolution of $Z_1(\cdot)$. Hence, $\mathsf{Succ}(N)$ can be re-defined by
      \begin{equation} \label{eq:succevent2}
        \mathsf{Succ}(N) \triangleq \left\{ \lim_{j \rightarrow \infty} \bar{Z}_1(j) = N \right\}
      \end{equation}
      such that 
      \begin{align*}
        &\mathbb P \left( \lim_{r \rightarrow \infty} Z_1(r)=N \mid  Z_1(0)=\zeta_1 N \right) \nonumber\\
        &= \mathbb P \left( \lim_{j \rightarrow \infty} \bar{Z}_1(j)=N \mid  \bar{Z}_1(0)=\zeta_1 N \right).
      \end{align*} 
      Furthermore, we define a ``worse'' random process $\bar{Z}'_1(\cdot)$ based on $\bar{Z}_1(\cdot)$ as follows
      \begin{align} \label{eq:simpletransit}
      \begin{cases}
        \mathbb{P} \left( \bar{Z}'_1(j) = \bar{Z}'_1(j-1)+1 \right) = 1 - \frac{2}{1+2 N \varepsilon^2} \vspace{1ex}\\
        \mathbb{P} \left( \bar{Z}'_1(j) \leq \bar{Z}'_1(j-1)-1 \right) = \frac{2}{1+2 N \varepsilon^2}
      \end{cases}
      \end{align}
      such that 
      \begin{align}
        &\mathbb{P} \left(\lim_{r\rightarrow\infty} \bar{Z}_1(r) = N \mid \bar{Z}_1(0) = \zeta_1 N\right)  \nonumber\\
        &\geq \mathbb{P}\left(\lim_{j\rightarrow\infty} \bar{Z}'_1(j) = N \mid \bar{Z}'_1(0) = \zeta_1 N\right)
      \end{align}
      In another word, $\bar{Z}'_1(\cdot)$ is defined by inducing an artificial condition $\bar{Z}_1(j) - \bar{Z}_1(j-1) \leq 1$ to $\bar{Z}_1(\cdot)$.
      Then, our aim becomes deriving a lower bound on $\mathbb P \left( \mathsf{Succ}(N) \mid \bar{Z}'_1(0)=\zeta_1 N \right)$. The lower bound should hold for any transition probability distribution of $\bar{Z}'_1(\cdot)$ respecting the conditions shown in (\ref{eq:simpletransit}). In the following, for notation convenience, let $\Delta\bar{Z}'_1(j) = \bar{Z}'_1(j) - \bar{Z}'_1(j-1)$, $q = \mathbb{P}\left(\Delta\bar{Z}'_1(j)=1 \mid \bar{Z}'_1(j)=\zeta_1 N \right)$ and $P_{\zeta_1 N} = \mathbb P \left( \mathsf{Succ}(N) \mid \bar{Z}'_1(0)=\zeta_1 N \right)$. Then,
      \begin{align} \label{eq:recurrence00}
        P_{\zeta_1 N}=&\mathbb{P} \left( \mathsf{Succ}(N) \mid \bar{Z}'_1(0) = \zeta_1 N \right)  \nonumber\\
        =& \mathbb{P} \left(\Delta \bar{Z}'_1(1) =1 \mid \bar{Z}'_1(0) = \zeta_1 N \right)  \nonumber\\
         & \cdot \mathbb{P} \left( \mathsf{Succ}(N) \mid \Delta \bar{Z}'_1(1) =1, \bar{Z}'_1(0) = \zeta_1 N \right)  \nonumber\\
         & +\mathbb{P} \left( \Delta \bar{Z}'_1(1) \leq -1 \mid \bar{Z}'_1(0) = \zeta_1 N \right)  \nonumber\\
         & \cdot \mathbb{P} \left( \mathsf{Succ}(N) \mid \Delta \bar{Z}'_1(1) \leq -1, \bar{Z}'_1(0) = \zeta_1 N \right)   \nonumber\\
        =& q \cdot \mathbb{P} \left( \mathsf{Succ}(N) \mid  \bar{Z}'_1(1)=\zeta_1 N+1 \right)  \nonumber\\
         & + (1-q) \cdot \mathbb{P} \left( \mathsf{Succ}(N) \mid  \bar{Z}'_1(1)\leq \zeta_1 N -1 \right) \nonumber\\
        \leq& q P_{\zeta_1 N+1}+ (1-q) P_{\zeta_1 N-1}
      \end{align}
      where the inequality holds since 
      \begin{align*}
        &\mathbb{P} \left( \mathsf{Succ}(N) \mid  \bar{Z}'_1(j) \leq \zeta_1 N - 1 \right) \\
        &\leq \mathbb{P} \left( \mathsf{Succ}(N) \mid \bar{Z}'_1(j)=\zeta_1 N-1 \right)
      \end{align*}
      By substituting ${P}_{\zeta_1 N} = q {P}_{\zeta_1 N} + (1-q) {P}_{\zeta_1 N}$ into (\ref{eq:recurrence00}), we have
      \begin{align}
        {P}_{\zeta_1 N+1} - {P}_{\zeta_1 N} \geq \frac{1-q}{q} \left( {P}_{\zeta_1 N} - {P}_{\zeta_1 N-1} \right)
      \end{align}
      Since $\mathbb P \left( \Delta\bar{Z}'_1(j) \geq 1 \mid \bar{Z}'_1(j-1) < \zeta'_1 N \right) \leq q$ for any $\zeta'_1 < \zeta_1$, by recurrence,
      \begin{equation}
        {P}_{\zeta_1 N+1} - {P}_{\zeta_1 N} \geq ({P}_1 - {P}_0) \left(\frac{1-q}{q}\right)^{\zeta_1 N}
      \end{equation}
      Since we hereby consider each agent has a non-null adoption, we have ${P}_0 = 0$ due to \textbf{Proposition}~\ref{prop:local}. Therefore, the above inequality can be re-written as
      \begin{equation}
        {P}_{\zeta_1 N+1} - {P}_{\zeta_1 N} \geq {P}_1 \left(\frac{1-q}{q}\right)^{\zeta_1 N}
      \end{equation}
      Furthermore,
      \begin{align}
        {P}_{\zeta_1 N+1} - {P}_{1} \geq \sum^{\zeta_1 N}_{\ell=1} ({P}_{\ell+1} - {P}_{\ell}) = {P}_1 \sum^{\zeta_1 N}_{\ell=1} \left(\frac{1-q}{q}\right)^\ell  
      \end{align}
      Therefore,
      \begin{align}
        {P}_{\zeta_1 N+1} \geq {P}_1\sum^{\zeta_1 N}_{\ell=0} \left(\frac{1-q}{q}\right)^\ell   = {P}_1 \frac{1 - \left(\frac{1-q}{q}\right)^{\zeta_1 N+1}}{1 - \frac{1-q}{q}} 
      \end{align}
      When $\zeta_1 N = N-1$, $P_N \geq \frac{1 - \left(\frac{1-q}{q}\right)^{N}}{1 - \frac{1-q}{q}} {P}_1$. Since ${P}_N=1$, 
      \begin{equation} \label{eq:probsucc1}
        {P}_1 \leq \frac{1 - \frac{1-q}{q}}{1-\left(\frac{1-q}{q}\right)^{N}}
      \end{equation}
      Note that the equality in (\ref{eq:probsucc1}) holds with $\mathbb{P} \left( \bar{Z}'_1(j) = \bar{Z}'_1(j-1)-1 \right) = 1-q$, which also satisfies the condition stated in (\ref{eq:simpletransit}). Therefore, we have
      \begin{equation*} 
        P_{\zeta_1 N+1} \geq \frac{1- \left( \frac{1-q}{q} \right)^{\zeta_1 N + 1}}{1-\left( \frac{1-q}{q} \right)^N}
      \end{equation*}
      Since $N > \frac{3}{2\varepsilon^2}$, we have $\frac{1}{2} < q \leq 1$ and thus $0 < 1-\left( \frac{1-q}{q} \right)^N \leq 1$. Therefore, we get
      \begin{align} 
        P_{\zeta_1 N} \geq& \frac{1- \left( \frac{1-q}{q} \right)^{\zeta_1 N}}{1-\left( \frac{1-q}{q} \right)^N} \geq 1- \left( \frac{1-q}{q} \right)^{\zeta_1 N} \nonumber\\
        =& 1 - \left( \frac{2}{2N \varepsilon^2-1} \right)^{\zeta_1 N}
      \end{align}   
      which completes the proof.
    \end{proof} 
    \begin{remark} \label{re:nfixed}
      It is revealed in \textbf{Theorem}~\ref{thm:learnability} that, given an initial condition $Z_1(0) = \zeta_1 N$, when the number of agents, i.e., $N$, is sufficiently large (with respect to constants $\varepsilon$ and $h$), $\mathsf{Succ}(N)$ holds with high probability. Another interesting question is, given fixed $N$, with what initial condition, the learnability can be ensured? \textbf{Theorem}~\ref{thm:learnability} also give an implication to answer this question. Specifically, given fixed $N$ and $h \leq \frac{N}{8 \log N}$, let $\varepsilon$ is a constant such that $\varepsilon > \sqrt{\frac{3}{2N}}$ and $\varepsilon \geq \frac{2\sqrt{6}}{h}$. When $\zeta_1 \in (0,1)$ satisfies $\zeta_1(1-\zeta_1) > {2\varepsilon}/{(p_1 - p_2)}$, we have (\ref{eq:learnability}) holds. 
    \end{remark}

    As shown in \textbf{Theorem}~\ref{thm:learnability}, the success event $\mathsf{Succ}(N)$ happens with high probability under an initial condition $Z_1(0) = \zeta_1 N$, whereas our algorithm begins with a more unified initial condition that all agents having no preference such that $Z_1(0)=0$ as demonstrated in Sec.~\ref{sec:algo}. In the following \textbf{Lemma}~\ref{le:init}, we reveal the evolution of $Z_1(\cdot)$ in the early phase of our algorithm where $Z_0(r) \neq 0$ and thus extend \textbf{Theorem}~\ref{thm:learnability} to a more general initial condition.
    \begin{lemma} \label{le:init}
      In any round $r$ such that $Z_0(r-1) = \zeta_0 N$ where $1 \leq \zeta_0 N \leq N$, for any $\delta \in (0,1)$, we have
      \begin{align} \label{eq:init}
        &\mathbb P \left( Z_1(r) \geq (1-\delta) \frac{\mu p_1 \zeta_0 N}{K} ~\bigg|~ Z_0(r-1) = \zeta_0 N \right)  \nonumber\\
        &\geq 1 - \exp\left( - \frac{\mu p_1 \zeta_0 N \delta^2}{2K}\right)
      \end{align}
    \end{lemma}
    \begin{proof}
      Recall $\mathcal{S}_0(r-1) \subseteq \mathcal{N}$ denote the group of agents with null adoptions in round $r$ and thus $Z_0(r-1) = |\mathcal{S}_0(r-1)|$. According to \textbf{Algorithm}~\ref{alg:collearning}, in any round $r$ with $\mathcal{S}_0(r-1) \neq \emptyset$ (or $Z_0(r-1) \neq 0$), each agent $i \in \mathcal{S}_0(r-1)$ adopts the best arm $a_1$ through either uniform sampling or learning from its peers. In particular, when the agent $i$ adopts $a_1$ through the uniform sampling, we have $a_i(r)=a_1$ with probability $\frac{\mu}{K}$ in the sampling stage and $\omega_i \leftarrow a_1$ with probability $p_1$ in the adopting stage. Therefore, in any round $r$ with $Z_0(r-1) \neq 0$, the probability for each agent $i \in \mathcal{S}_0(r-1)$ to adopt $a_1$ is at least $\frac{\mu p_1}{K}$. Hence, $\mathbb{E}\left[\sum_{i\in \mathcal{S}_0(r-1)} X_{i,1}(r)\right] \geq \frac{\mu p_1 \zeta_0 N}{K}$. According to the Chernoff-Hoeffding bound~\cite{DubhashiP-book09}, we have
      \begin{align*}
        &\mathbb{P}\left( \sum_{i\in \mathcal{S}_0(r)} X_{i,1}(r) \geq \frac{\mu p_1 \zeta_0 N(1-\delta) }{K} ~\bigg|~ Z_0(r-1) = \zeta_0 N \right)  \nonumber\\
        &\geq 1 - \exp\left( - \frac{\mu p_1 \zeta_0 N \delta^2}{2K}\right)
      \end{align*}
      which completes the proof since $Z_1(r) \geq \sum_{i\in \mathcal{S}_0(r)} X_{i,1}(r)$.
    \end{proof}

  \subsection{Reliability } \label{ssec:reliability}
    We now look at the reliability of our proposed collaborative learning algorithm in face of agent corruptions. We let $\mathcal{N}^\dagger \subseteq \mathcal{N}$ denote the subset of $\tau N$ corrupted agents. For notation convenience, we re-define some notations and symbols used to analyze the learnability of our algorithm with no agent corruptions considered. We denote by $Z_k(r)$ the number of honest agents adopting arm $a_k$ in round $r$ and let $Q_k(r) = \frac{Z_k(r)}{(1-\tau)N}$ denote the proportion of the honest agents adopting $a_k$ in round $r$. We also denote by $Z^+_1(r)$ the number of honest agents who adopt $a_{\geq 2}$ in round $r-1$ and $a_1$ in round $r$, and let  $Z^-_1(r)$ be the number of honest agents who adopt $a_{1}$ in round $r-1$ and $a_{\geq 2}$ in round $r$. Similarly, let $\mathcal X(r)$ denote the adoption state of the honest agents in round $r$.

    It is worthy to note that the learnability should be re-explained in this case. When there are a fixed set of corrupted agents, their falsified recommendations may “deceive” some of the honest agents to adopt non-optimal arms. For example, as shown in \textbf{Lemma}~\ref{le:upbd}, there always are a fraction of honest agents adopting the non-optimal arms on expectation, when $N \rightarrow \infty$. Furthermore, it is demonstrated by \textbf{Lemma}~\ref{le:finitedecrease} that, when there are a large number of agents participating in the collaborative learning process, at lest one honest agent adopting $a_1$ in round $r-1$ would take $a_{\geq 2}$ as adoption in round $r$, even all agents already adopted the best arm $a_1$ in round $r-1$. Therefore, we reveal the reliability of our collaborative learning algorithm by illustrating the evolution of $Z_1(\cdot)$, as shown in \textbf{Theorem}~\ref{thm:infreliability} and \textbf{Theorem}~\ref{thm:finitereliability}. 
    \begin{lemma} \label{le:upbd}
      Suppose $\tau N$ agents are corrupted and the number of arms $K\geq 2$. When $N \rightarrow \infty$, in any round $r$, we have
      \begin{equation} \label{eq:upbdq1}
        \mathbb{E}[Q_1(r)] \leq 1 - \frac{\tau p_K}{K}
      \end{equation}
    \end{lemma}
    \begin{proof}
      It is well known that $\mathbb E[Q_1(r))] \leq \mathbb E[Q_1(r) \mid Z_1(r-1) = (1-\tau)N)]$ in any round $r$, since $Z_1(r) \leq (1-\tau) N$ holds for any $r$. Hence, we prove this lemma by deriving an upper bound on $\mathbb E[Q_1(r) \mid Z_1(r-1) = (1-\tau)N)]$. Let $\widetilde Q_k(r)$ denote the popularity of $a_k$-token disseminated in round $r$. In any round $r$ with $Z_1(r-1) = (1-\tau)N$, all $a_{\geq 2}$-tokens are from the corrupted agents. Since each of them chooses an arbitrary arm uniformly at random, we have $\mathbb E \left[ \widetilde{Q}_k(r) \big| Z_1(r-1) = (1-\tau)N \right] = {\tau}/{K}$ for any $k=2,\cdots,K$. According to \textbf{Lemma}~\ref{le:sampling}, when $N \rightarrow \infty$, an honest agent adopts $a_{\geq 2}$ with probability $\sum^K_{k=2} \widetilde{Q}_k(r) p_k$. Therefore, we complete the proof by $\mathbb E[ Q_1(r) \mid Z_1(r-1) = (1-\tau)N ] = 1 - \mathbb E \left[\sum^K_{k=2} \widetilde{Q}_k(r) p_k\right]\leq 1 - \frac{\tau}{K}\sum^K_{k=2}p_k$. Since $K \geq 2$ and $p_2 \geq \cdots \geq p_K$, we have $1-\frac{\tau}{K}\sum^K_{k=2}p_k \leq 1 - \frac{\tau p_K}{K}$. 
    \end{proof}

    \begin{lemma}  \label{le:finitedecrease}
      For any $0 \leq \varepsilon \leq \min \left\{ \frac{(1-\tau)Np_K}{4((1-\tau)Np_K+1)}, \sqrt{\frac{(1-\tau)\tau}{5K}} \right\}$, in any round $r$, we have $Z^-_1(r) \geq 1$ with probability at least $1-\frac{3}{1+2(1-\tau)N\varepsilon^2}$ when $N \geq 9$ is sufficiently large such that $N \geq 8h \log N$ and $h \geq \frac{24}{\varepsilon^2}$.
    \end{lemma}
    \begin{proof}
      Given a fixed set of $\tau N$ agents, it is well known that $\mathbb P (Z^-_1(r) \geq 1) \geq \mathbb P (Z^-_1(r) \geq 1 \mid Z_1(r-1)=(1-\tau)N)$. Hence, we prove this lemma by revealing the lower bound of $\mathbb P (Z^-_1(r) \geq 1 \mid Z_1(r-1)=(1-\tau)N)$. Therefore, the remaining of the proof is conducted with condition $Z_1(r-1) = (1-\tau)N$. Let $M'_k(r)$ denote the number of corrupted agents choosing arm $a_k$ to disseminate in round $r$. Since each corrupted agent chooses an arbitrary arm uniformly at random to ``deceive'' the honest agents, according to the Chernoff-Hoeffding bound \cite{DubhashiP-book09}, we have
      \begin{align}
        \mathbb P \left( M'_1(r) \leq \frac{3\tau N}{2K} \right) \geq 1 - \exp\left(-\frac{\tau N}{10K}\right)
      \end{align}
      and thus
      \begin{align} \label{eq:finitedecrease00}
        \mathbb P \left( \widetilde{Q}_1(r) \leq 1 - \left(1-\frac{3}{2K}\right)\tau \right) \geq 1 -  \exp\left(-\frac{\tau N}{10K}\right)
      \end{align}
      According to the the Chernoff-Hoeffding bound \cite{DubhashiP-book09} again, for any $0 \leq \varepsilon \leq 1 - \widetilde{Q}_1(r)$, we have
      \begin{align}
        Z^-_1(r) \geq& \sum^{(1-\tau)N}_{i=1} \sum^K_{k=2} \widetilde{Q}_{i,k}(r)p_k - \varepsilon  \nonumber\\
        \geq& p_K \sum^{(1-\tau)N}_{i=1} \left( 1-\widetilde{Q}_{i,1}(r) \right) - \varepsilon
      \end{align}
      hold with probability at least $1 - \exp(-2(1-\tau)N \varepsilon^2)$. By combining (\ref{eq:finitedecrease00}) with the above inequality, we then obtain
      \begin{align*}
        Z^-_1(r) \geq (1-\tau)Np_K \left( 1- \widetilde{Q}_1(r) \right) - \left( (1-\tau)Np_K+1 \right)\varepsilon
      \end{align*}
      hold with probability at least $1 - \frac{3}{N^2} - \frac{1}{1+2(1-\tau)N \varepsilon^2}$, when $N\geq 9$ is sufficiently large such that $N \geq 8h \log N$ and $h \geq \frac{24}{\varepsilon^2}$. Therefore, for any sufficiently small $\varepsilon$ such that $0 \leq \varepsilon \leq \min \left\{ \frac{(1-\tau)Np_K}{4((1-\tau)Np_K+1)}, \sqrt{\frac{(1-\tau)\tau}{5K}} \right\}$, we have $\mathbb P\left(Z^-_1(r) > 0 \right) \geq 1-\frac{3}{1+2(1-\tau)N\varepsilon^2}$, which completes the proof since $Z^-_1(r)$ is a positive integer.
    \end{proof}

    In \textbf{Theorem}~\ref{thm:infreliability}, we show that $\mathbb{E}[Z_1(r) - Z_1(r-1) \mid \mathcal X(r-1)] = \mathbb{E}[Z^+_1(r) - Z^-_1(r) \mid \mathcal X(r-1)] \geq 0$ when $N \rightarrow \infty$ but the proportion of the corrupted agents, i.e., $\tau$, is bounded.
    \begin{theorem} \label{thm:infreliability}
      Suppose there are $\tau N$ corrupted agents. Let $0<\alpha<1$. When $N \rightarrow \infty$, if $\tau \leq \frac{(1-\alpha)(p_1-p_2)}{(1-\alpha)p_1 + \alpha p_2}$, we have 
      \begin{equation} \label{eq:infreliability}
        \mathbb{E}(Z_1(r) - Z_1(r-1) \mid \mathcal{X}(r-1)) \geq 0
      \end{equation}
      in any round $r$ such that $0 < Q_1(r-1) \leq \alpha$.
    \end{theorem} 
    \begin{proof}
      Recall $\widetilde Q_k(r)$ denotes the popularity of $a_k$-token disseminated in round $r$. Since the corrupted nodes choose an arbitrary arm to disseminate uniformly at random, we have
      \begin{align} \label{eq:corruptineq}
        \widetilde{Q}_k(r) \geq (1-\tau)Q_k(r-1), ~\forall k=1,\cdots,K
      \end{align}
      According to \textbf{Lemma}~\ref{le:sampling}, when $N \rightarrow \infty$, each agent $i$ chooses arm $a_k$ with probability $\widetilde{Q}_k(r)$ in the sampling stage of round $r$. The expected values of $Z^-_1(r)$ and $Z^-_1(r)$ can be represented by
      \begin{align}
        &\mathbb E[Z^+_1(r) \mid \mathcal X(r-1)] = ((1-\tau)N - Z_1(r-1))\widetilde{Q}_1(r) p_1 \nonumber\\
        &\geq ((1-\tau)N - Z_1(r-1))(1-\tau){Q}_1(r) p_1
      \end{align}
      and
      \begin{align}
        &\mathbb E[Z^-_1(r) \mid \mathcal X(r-1)] = Z_1(r-1)\sum^K_{k=2} \widetilde{Q}_k(r)p_k \nonumber\\
        &\geq  (1-\tau)(N - Z_1(r-1)){Q}_1(r)p_2
      \end{align}
      respectively. Therefore,
      \begin{align}
        \frac{\mathbb E[ Z^+_1(r) \mid \mathcal X(r-1)]}{\mathbb E[ Z^-_1(r) \mid \mathcal X(r-1)]} \geq  \frac{p_1}{p_2} \cdot \frac{1-\tau}{1 + \frac{\tau}{1/Q_1(r-1) - 1}}
      \end{align}
      When $0 < Q_1(r-1) \leq \alpha$, we continue the above inequality and obtain
      \begin{align}
        \frac{\mathbb E[ Z^+_1(r) \mid \mathcal X(r-1)]}{\mathbb E[ Z^-_1(r) \mid \mathcal X(r-1)]} \geq  \frac{p_1}{p_2} \cdot \frac{(1-\alpha)(1-\tau)}{1 - \alpha + \alpha\tau}
      \end{align}
      We finally have $\mathbb E[Z^+_1(r) \mid \mathcal X(r-1)] \geq \mathbb E[Z^-_1(r)  \mid \mathcal X(r-1)]$ and thus complete the proof when $\tau \leq \frac{(1-\alpha)(p_1-p_2)}{(1-\alpha)p_1 + \alpha p_2}$.
    \end{proof}

    We also demonstrate the reliability of our algorithm with a finite but sufficient number of agents participating in the collaborative learning process.
    \begin{theorem} \label{thm:finitereliability}
      Let $0 < \alpha'<\alpha<1$ be constants. Let $\tau < \frac{(1-\alpha)(p_1-p_2)}{(1-\alpha)p_1+\alpha p_2}$ and assume $g(x) = -(1-\tau)(p_1-p_2)x^2 + ((1-\tau)p_1 - p_2)x$ is a quadratic function. In any round $r$ such that $\alpha' \leq Q_1(r-1) \leq \alpha$, we have
      \begin{align}
        &\mathbb P[ Z_1(r) - Z_1(r-1) \geq 1 \mid \mathcal X(r-1)]  \nonumber\\
        \geq& 1 - \frac{2}{1+2(1-\tau)N\varepsilon^2}
      \end{align}
      where $0 < \varepsilon < \min\{g(\alpha'), g(\alpha)\}$, when $N$ is sufficiently large such that $\frac{\log N}{N} \leq \frac{1}{8h}$ and $h \geq \frac{24}{\varepsilon^2}$.
    \end{theorem}
    \begin{proof}
      According to \textbf{Lemma}~\ref{le:lwtail4Qik}, in each round $r$, for any $k=1,2,\cdots,K$, we have
      \begin{align} \label{eq:advqik}
        &\mathbb{P} \left(Q_{i,k}(r) \geq \widetilde{Q}_k(r)-\varepsilon \mid \mathcal X(r-1), \forall i\right)  \nonumber\\
        \geq& 1 - \frac{N^{1-h\varepsilon^2/8}}{1 - {4h N^{-1} \log{N}}} - \frac{1}{N^2}, ~~0 \leq \varepsilon \leq \widetilde{Q}_k(r)
      \end{align}
      We denote by $\mathcal S_1(r-1) = \left\{i \in \mathcal{N}^\dagger : X_{i,1}(r-1)=1 \right\}$ and $\mathcal S_{\geq 2}(r-1) = \left\{i \in \mathcal{N}^\dagger : \sum^K_{k=2} X_{i,k}(r) = 1 \right\}$ the set of honest agents adopting arm $a_1$ and the set of ones adopting $a_{\geq 2}$, respectively. We then have $Z_1(r-1) = |\mathcal S_1(r-1)|$ and $Z_{\geq 2}(r-1) = |\mathcal S_{\geq 2}(r-1)|$. By applying the Chernoff-Hoeffding bound~\cite{DubhashiP-book09},
      \begin{align}
        \frac{Z_1(r)}{(1-\tau) N} \geq& \frac{1}{(1-\tau)N} \sum_{i \in \mathcal S_1(r-1)} \left( Q_{i,1}(r) + \sum^K_{k=2}Q_{i,k}(r)p_k \right)  \nonumber\\
        &+ \frac{1}{(1-\tau)N} \sum_{i \in \mathcal S_{\geq2}(r-1)} Q_{i,1}(r)p_1 - \varepsilon  \nonumber\\
        \geq& \frac{p_1 \sum_{i \in \mathcal S_{\geq 2}(r-1)} Q_{i,1}(r)}{(1-\tau)N}  + \frac{p_2 \sum_{i \in \mathcal S_{1}(r-1)} Q_{i,1}(r)}{(1-\tau)N}   \nonumber\\
        &+ \frac{(1-p_2)S_1(r-1)}{(1-\tau)N} - \varepsilon
      \end{align}
      holds with probability at least $1-\exp(-2(1-\tau)N\varepsilon^2)$ (with $0 \leq \varepsilon \leq \widetilde{Q}_k(r)$), where the second inequality holds since $p_2 \geq p_3 \geq \cdots \geq p_K$. By substituting (\ref{eq:advqik}) and (\ref{eq:corruptineq}) into the above inequality, we have
      \begin{align*} 
        &\frac{Z_1(r) - Z_1(r-1)}{(1-\tau) N} \nonumber\\
        \geq& \frac{p_1 S_{\geq 2}(r-1) \widetilde{Q}_1(r) + p_2 S_{1}(r-1) \widetilde{Q}_1(r) - p_2 S_1(r-1)}{(1-\tau)N}   \nonumber\\
        & - \frac{p_1 S_{\geq 2}(r-1) + p_2 S_{1}(r-1)}{(1-\tau)N} \varepsilon - \varepsilon \nonumber\\
        \geq& \frac{(1-(1-\tau)Q_1(r-1))(p_1-p_2)S_1(r-1) - \tau p_1 S_1(r-1)}{(1-\tau)N}   \nonumber\\
        & - \frac{p_1(1-\tau)N - (p_1-p_2)S_1(r-1)}{(1-\tau)N} \varepsilon - \varepsilon  \nonumber\\
        =& (1-(1-\tau)Q_1(r-1))Q_1(r-1)(p_1-p_2) - \tau p_1 Q_1(r-1)   \nonumber\\
        & - (p_1 - (p_1-p_2)Q_1(r-1)) \varepsilon - \varepsilon  \nonumber\\
        \geq& (1-(1-\tau)Q_1(r-1))Q_1(r-1)(p_1-p_2)  \nonumber\\
        &- \tau p_1 Q_1(r-1)  - 2\varepsilon
      \end{align*}
      hold with probability at least
      \begin{align}
        &1 - \frac{N^{1-h\varepsilon^2/8}}{1 - {4h N^{-1} \log{N}}} - \frac{1}{N^2} - \exp(-2(1-\tau)N\varepsilon^2)  \nonumber\\
        \geq& 1 - \frac{3}{N^2} - \exp(-2(1-\tau)N\varepsilon^2)  \nonumber\\
        \geq& 1 - \frac{2}{1+2(1-\tau)N\varepsilon^2}
      \end{align}
      where the first inequality holds when $N$ is sufficiently large such that $\frac{\log N}{N} \leq \frac{1}{8h}$ and $h \geq \frac{24}{\varepsilon^2}$, while the second one holds since ${N^2}/{3} \geq 3N \geq 1+2(1-\tau)N \varepsilon^2 \geq \exp(-2 (1-\tau) N \varepsilon^2)$ when $N \geq 9$.
      When $\tau < \frac{(p_1-p_2)(1-\alpha)}{p_1(1-\alpha)+p_2 \alpha}$, we have $(1-(1-\tau)Q_1(r-1))Q_1(r-1)(p_1-p_2) - \tau p_1 Q_1(r-1) > 0$ for any $Q_1(r-1) \in (0, \alpha]$. Furthermore, we let $\varepsilon$ be sufficiently small such that
      \begin{equation}
        0 < \varepsilon < \frac{1}{2} \max \{ g(\alpha'), ~g(\alpha) \}
      \end{equation}
      and thus $Z_1(r) - Z_1(r-1) > 1$.
    \end{proof}
    \begin{remark}
      According to our adversary model, when there are more agents corrupted, the proportion of the honest agents deceived by the adversary ones is larger. In another word, we have to be content with less tolerance for corrupted agents, if we expect to have a larger fraction of honest agents adopting the best arm. Specifically, as suggested in \textbf{Theorem}~\ref{thm:infreliability} and \textbf{Theorem}~\ref{thm:finitereliability} when $\alpha \rightarrow 1$, the proportion of the adversary agents tolerated by our collaborative learning algorithm approaches $0$.
    \end{remark}

\section{Numerical Results}  \label{sec:sim}
  In this section, we perform extensive simulations on both synthetic and real datasets in Sec.~\ref{ssec:synthetic} and Sec.~\ref{ssec:real}, respectively, to verify the efficacy of our algorithm. Throughout this section, we fix $h=64$ and $\mu=0.3$ for the disseminating stage and sampling stage of our collaborative learning algorithm, respectively, since they are constants which have a very slight impact on the performance of our algorithm. For each reported data points, we conduct the experiments thirty times to take an average on the results.

  \subsection{Simulations with Synthetic Data} \label{ssec:synthetic}
    We fist show the learnability of our algorithm. According to \textbf{Theorem}~\ref{thm:learnability}, our evaluation is performed by varying the number of agents and tuning the difference between $p_1$ and $p_2$. Specifically, we gradually increase the number of agents (i.e., $N$) from $2 \times 10^3$ to $6 \times 10^3$. We connect the agents randomly such that the resulting communication graph is connected and non-bipartite and each agent may have a very different number of neighbors in the graph. Additionally, we fixed $p_1=0.8$ and vary $p_2=0.6, 0.5, 0.4$. We also let the number of arms $K=100, 200, 300$ to investigate the impact of $K$ on the learnability of our algorithm, considering $K$ matters in the early phase of our algorithm (see \textbf{Lemma}~\ref{le:init}).
    \begin{figure*}[htb!]
      \begin{center}
        \parbox{.33\textwidth}{\center\includegraphics[width=.30\textwidth]{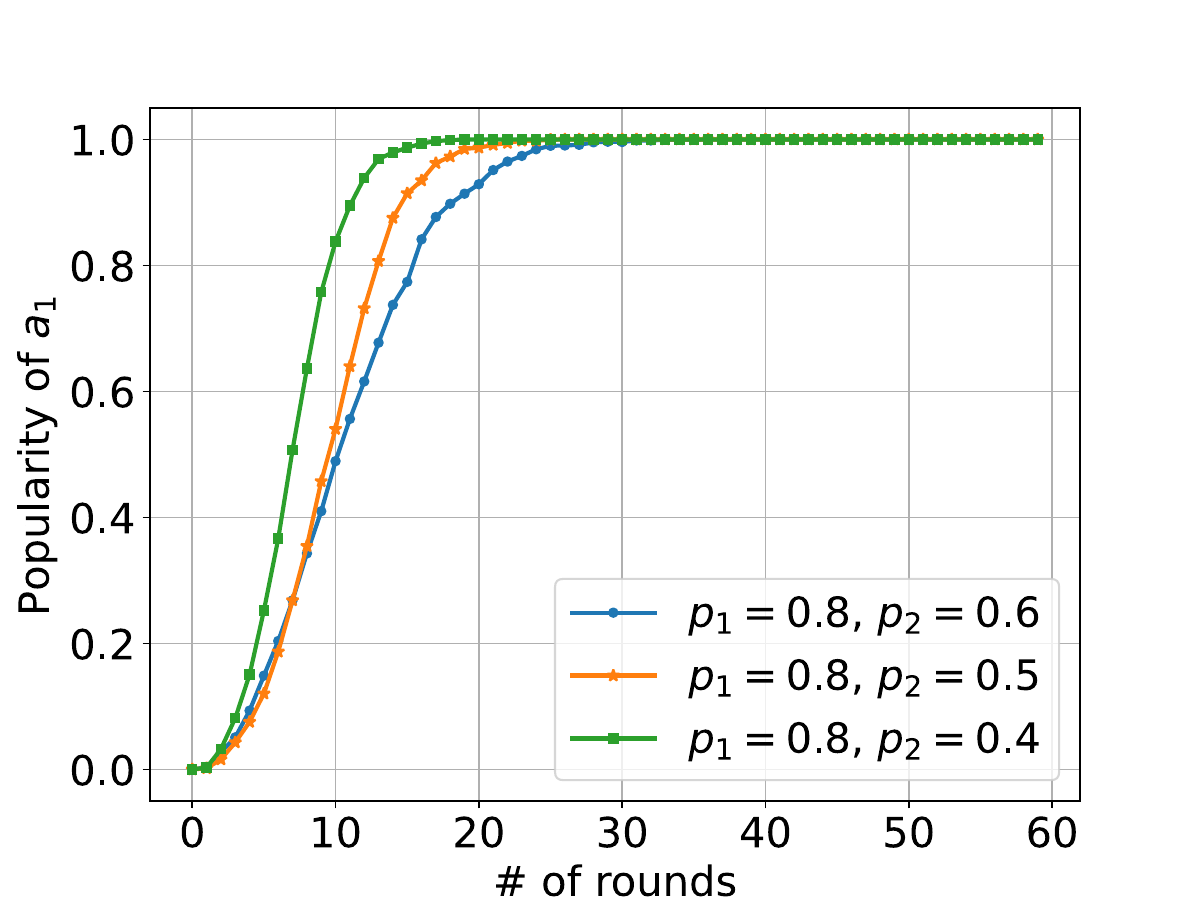}}
        \parbox{.33\textwidth}{\center\includegraphics[width=.30\textwidth]{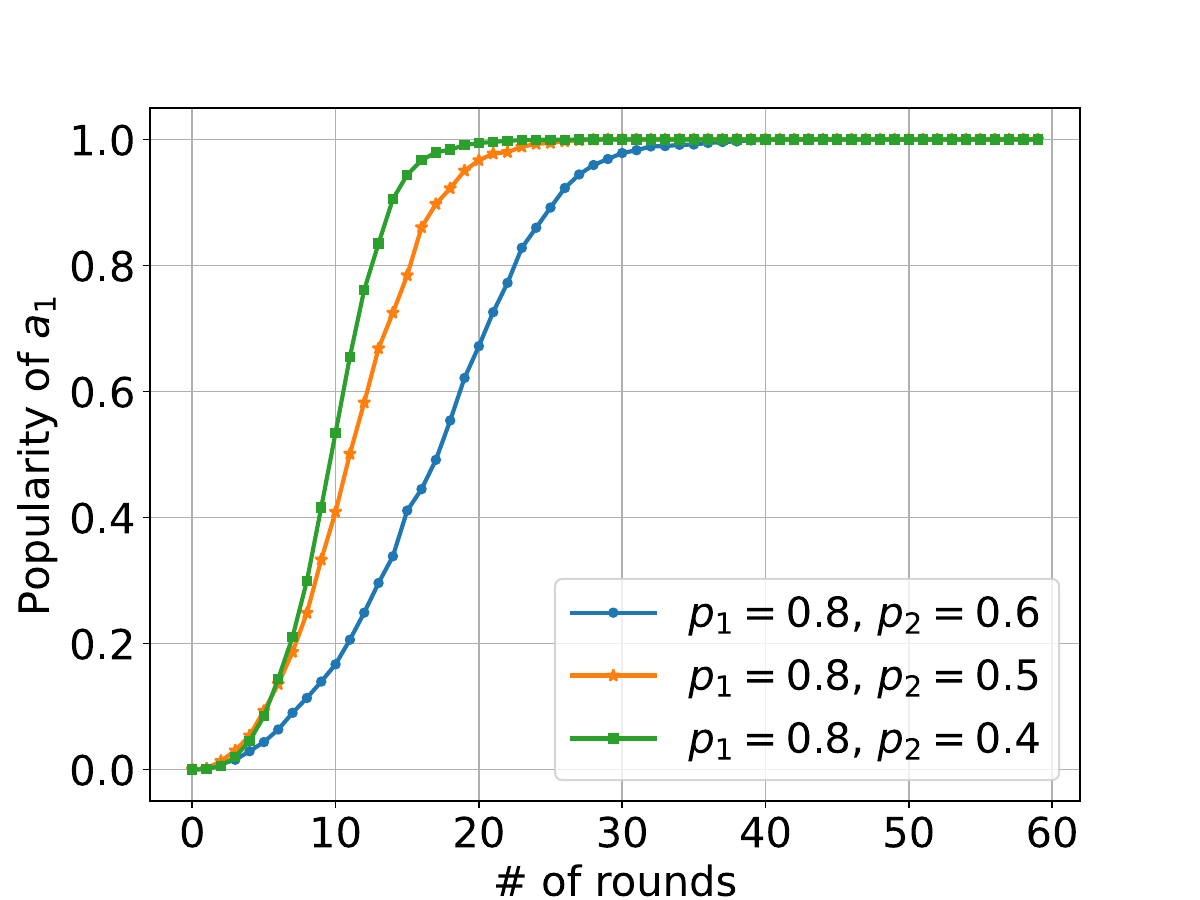}}
        \parbox{.33\textwidth}{\center\includegraphics[width=.30\textwidth]{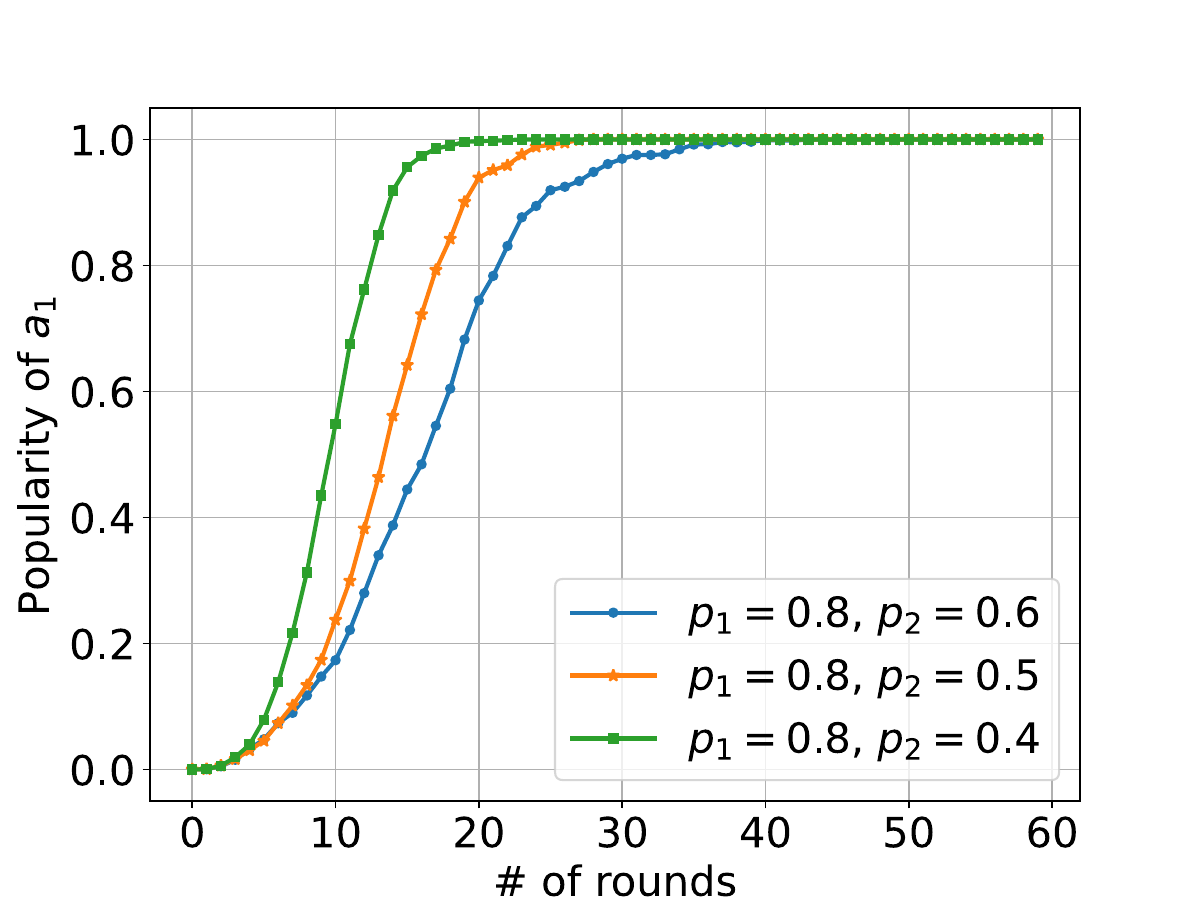}}
        \parbox{.33\textwidth}{\center\scriptsize(a) $N=2 \times 10^3, K=100$}
        \parbox{.33\textwidth}{\center\scriptsize(b) $N=2 \times 10^3, K=200$}
        \parbox{.33\textwidth}{\center\scriptsize(c) $N=2 \times 10^3, K=300$}
        \parbox{.33\textwidth}{\center\includegraphics[width=.30\textwidth]{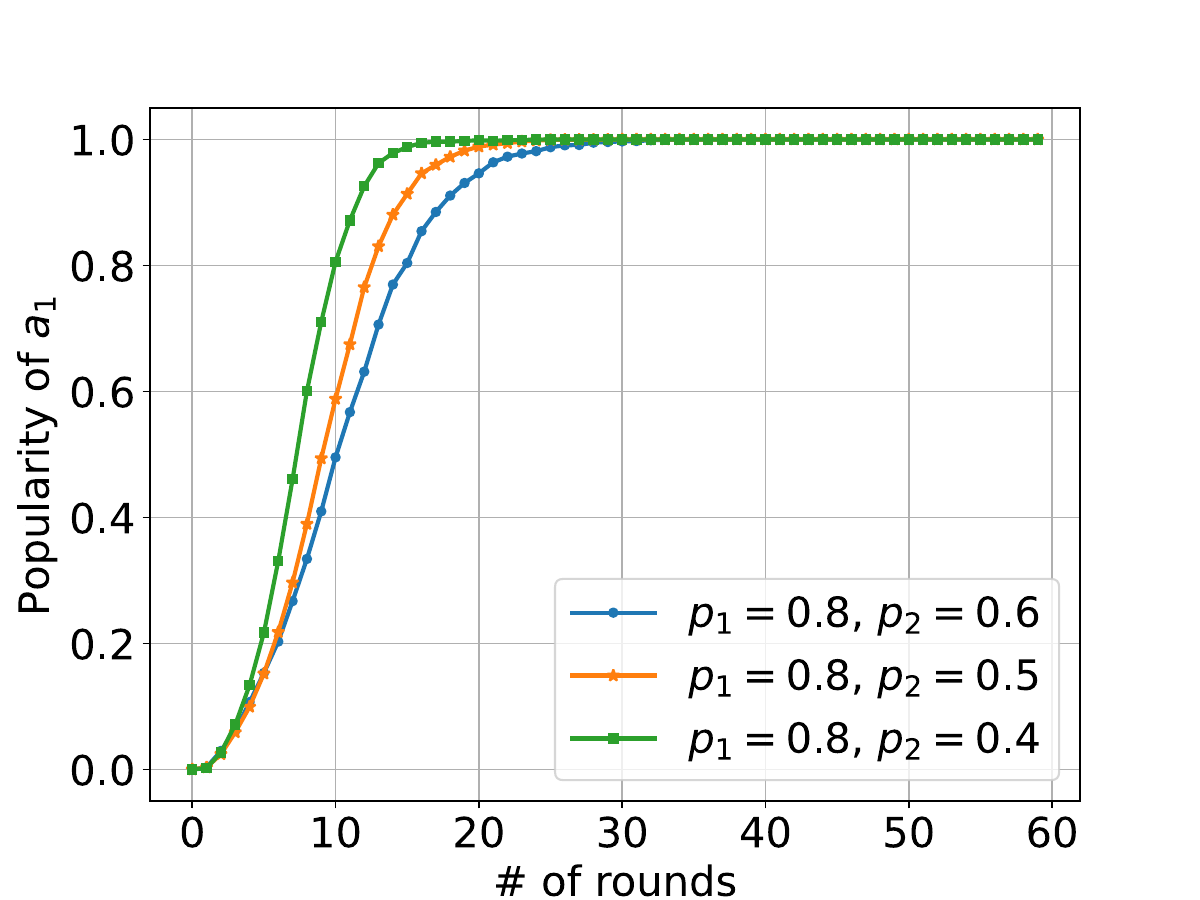}}
        \parbox{.33\textwidth}{\center\includegraphics[width=.30\textwidth]{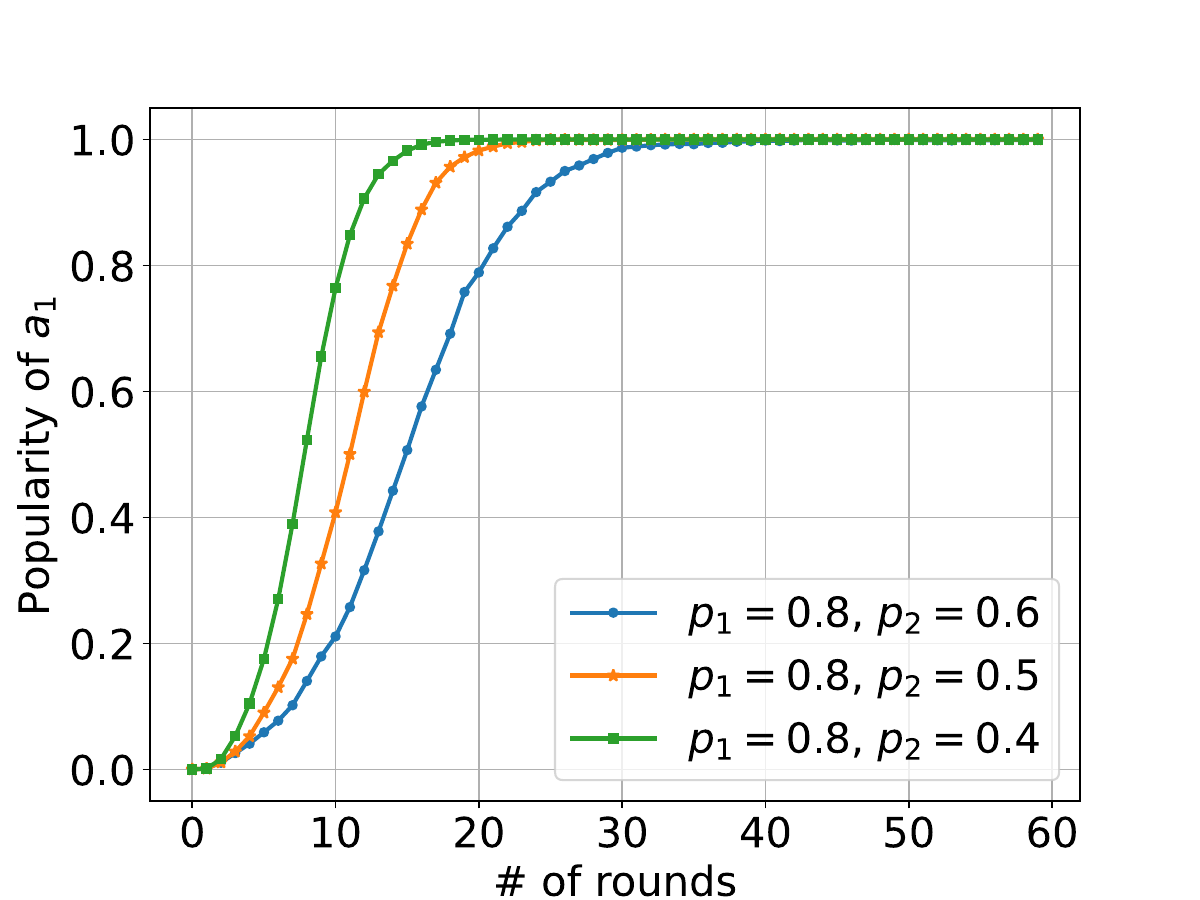}}
        \parbox{.33\textwidth}{\center\includegraphics[width=.30\textwidth]{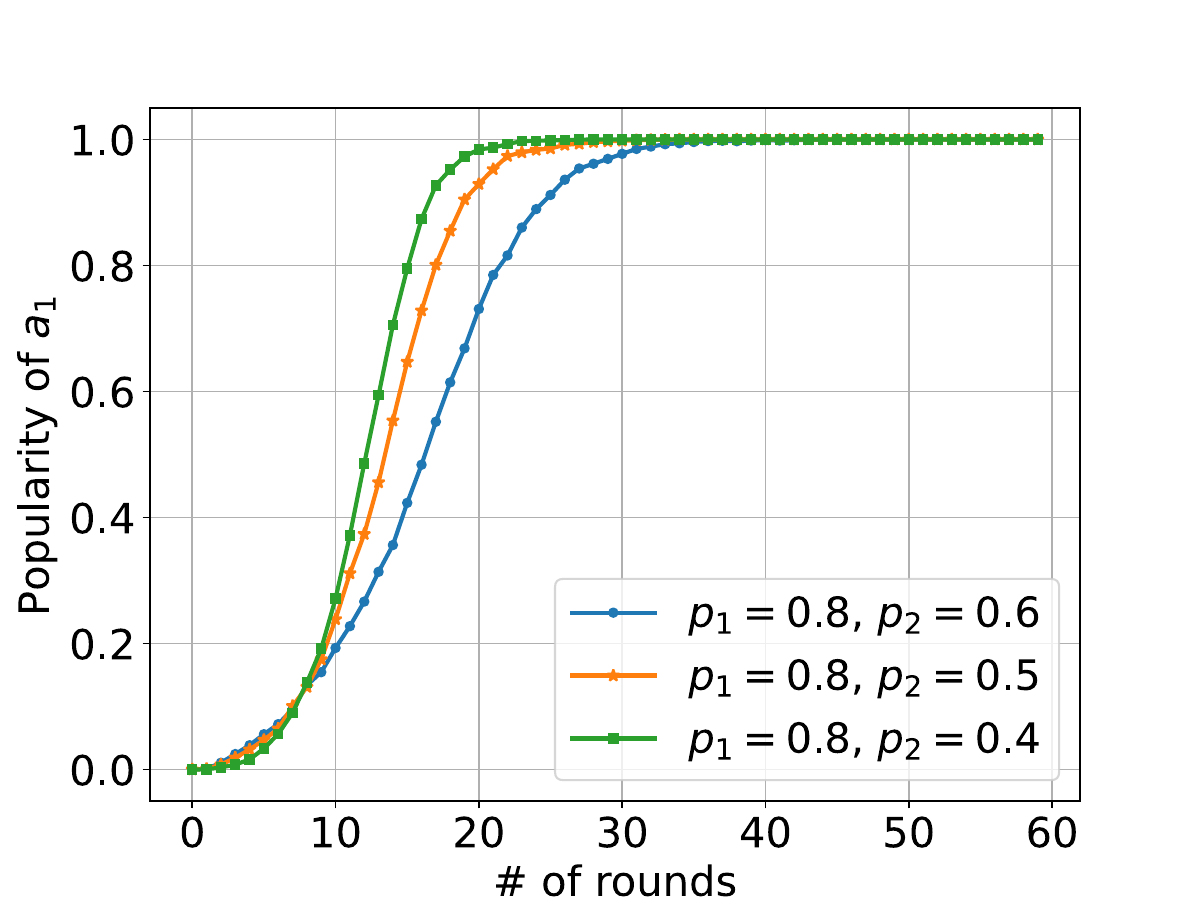}}
        \parbox{.33\textwidth}{\center\scriptsize(d) $N=4 \times 10^3, K=100$}
        \parbox{.33\textwidth}{\center\scriptsize(e) $N=4 \times 10^3, K=200$}
        \parbox{.33\textwidth}{\center\scriptsize(f) $N=4 \times 10^3, K=300$}
        \parbox{.33\textwidth}{\center\includegraphics[width=.30\textwidth]{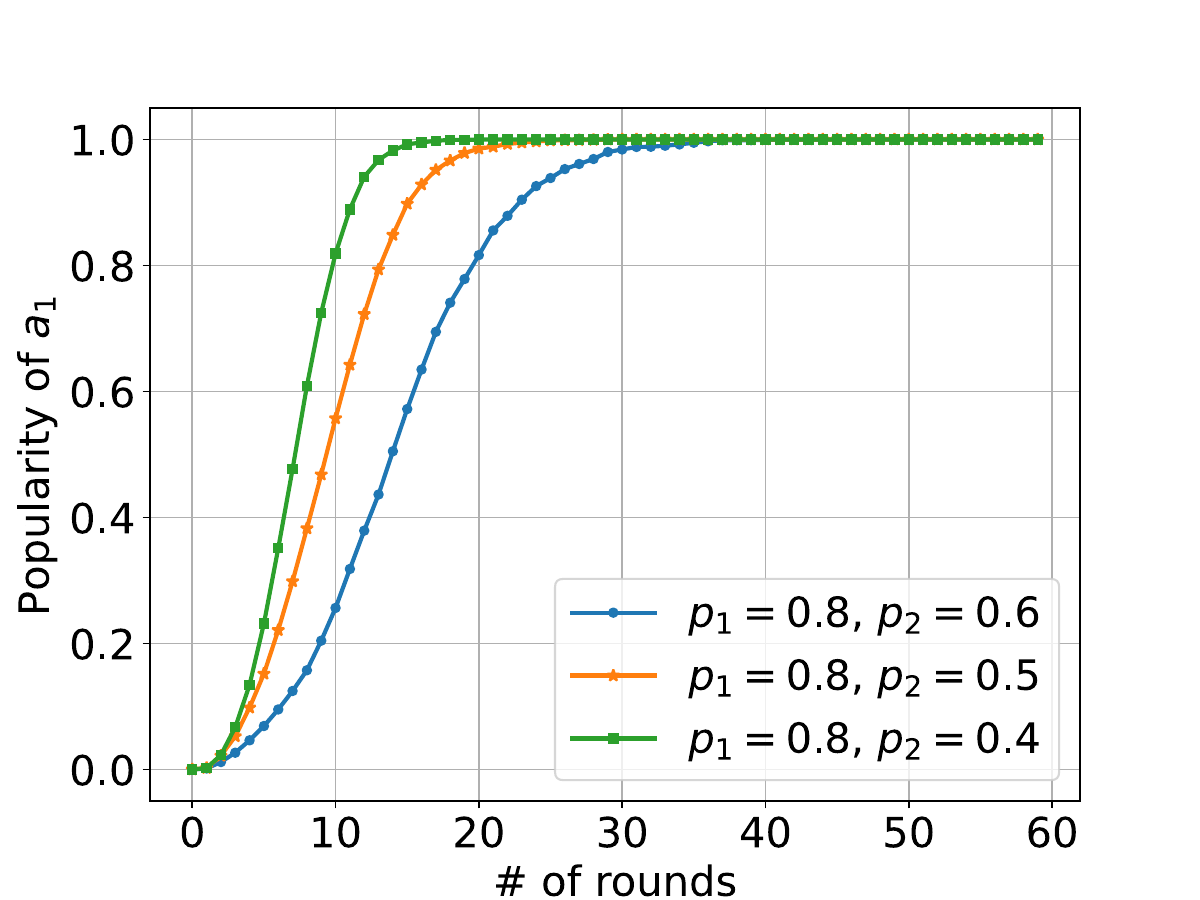}}
        \parbox{.33\textwidth}{\center\includegraphics[width=.30\textwidth]{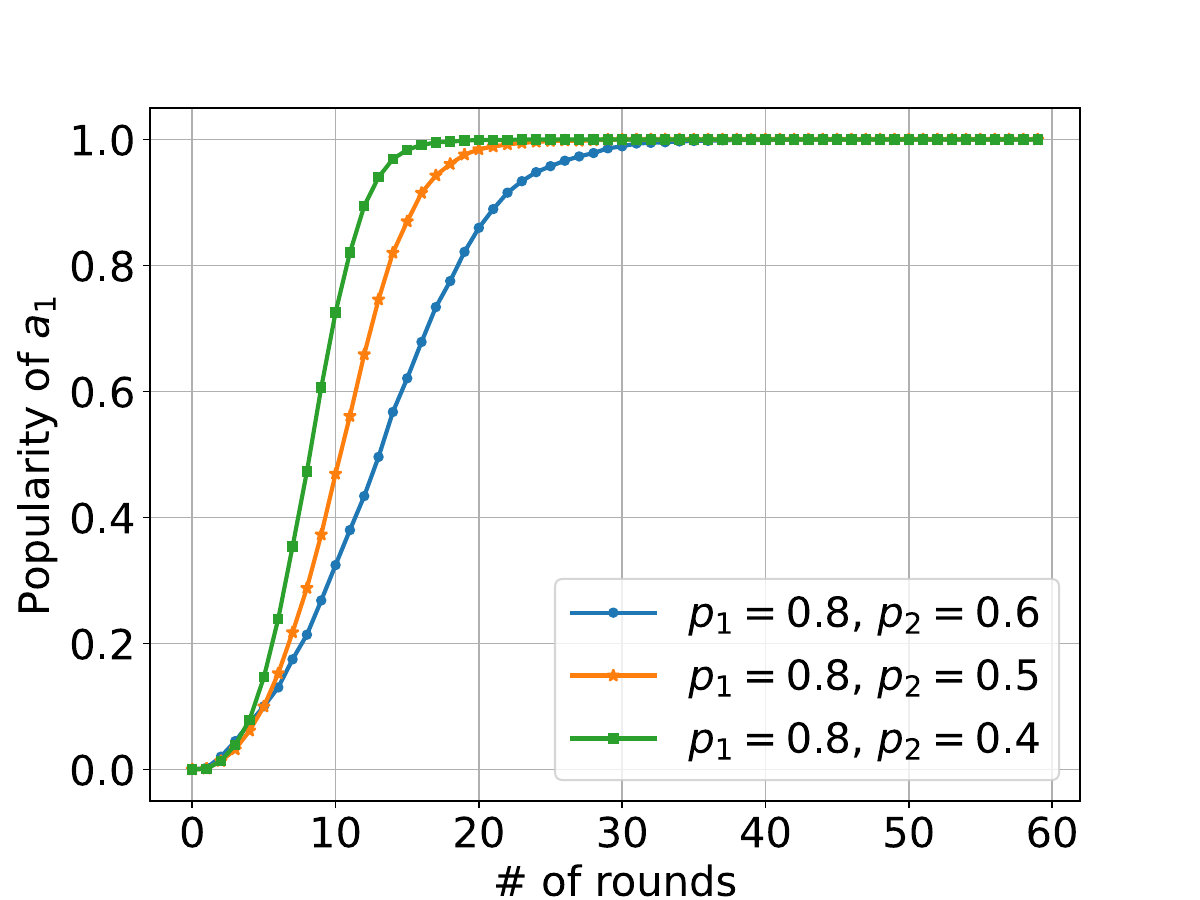}}
        \parbox{.33\textwidth}{\center\includegraphics[width=.30\textwidth]{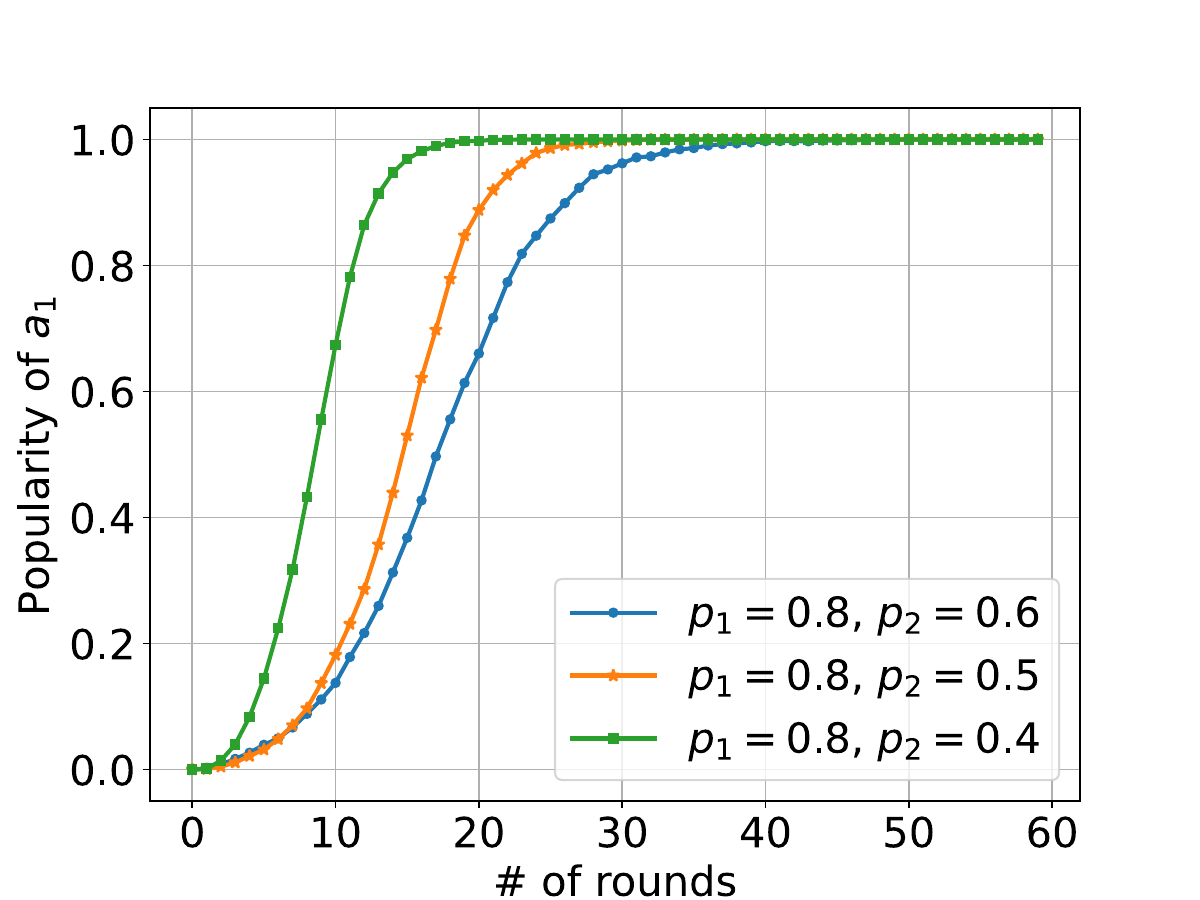}}
        \parbox{.33\textwidth}{\center\scriptsize(g) $N=6 \times 10^3, K=100$}
        \parbox{.33\textwidth}{\center\scriptsize(h) $N=6 \times 10^3, K=200$}
        \parbox{.33\textwidth}{\center\scriptsize(i) $N=6 \times 10^3, K=300$}
        \caption{Learnability under different settings with synthetic dataset.}
        \label{fig:learn}
      \end{center}
    \end{figure*}

    The numerical results are reported in Fig.~\ref{fig:learn}, where we use the popularity of the best arm $a_1$ (i.e., $Q_1(r)$) to illustrate the learning dynamics of our algorithm. It is demonstrated that, all of the agents eventually learn the best arm (i.e., $Q_1(r)=1$), according to \textbf{Theorem}~\ref{thm:learnability}. Although increasing the number of agents results in a slight increase in terms of the number of rounds, $\mathsf{Succ}(N)$ can be achieved in about 40 rounds in all of our settings. Note that the number of rounds actually is not meant to the actual temporal complexity. More time is necessitated to achieve $\mathsf{Succ}(N)$ when there are more agents, recalling each round consists of $\mathcal O(\log^2 N)$ slots for the agents to disseminate their adoptions as shown in \textbf{Theorem}~\ref{thm:complexity}. Nevertheless, the temporal cost is an inevitable investment to ensure the learnability in large-scale multi-agent graphs. Furthermore, the number of the agents adopting $a_1$ approaches $N$ with a higher rate when there is a larger gap between $p_1$ and $p_2$, which is consistent with what has been mentioned in \textbf{Remark}~\ref{re:learningrate}. Additionally, as implied by \textbf{Theorem}~\ref{thm:learnability} and \textbf{Lemma}~\ref{le:init}, the number of arms, i.e., $K$, actually has a very slight impact on the learnability of our collaborative learning algorithm.

    We then evaluate the reliability of our algorithm under different settings. We let $N=1\times 10^3$ and the generate the edges among the agents randomly and adopt the same setting on $N$ and $K$ as what we did in the last experiments. We vary $\alpha$ from $0.9$ to $0.6$ with a step size $0.1$ and let $\tau$ take its upper bound, i.e., $\tau = \frac{(1-\alpha)(p_1-p_2)}{(1-\alpha)p_1 + \alpha p_2}$. The results (i.e., the evolution of the popularity of the best arm among the honest agents) are reported in Fig.~\ref{fig:synreliability}. It is shown that a smaller $\alpha$ implies a smaller upper bound of the popularity of $a_1$, since we have the popularity of $a_1$ in round $r$ (i.e., $Q_1(r)$) increased when $Q_1(r-1) \leq {\alpha}$ according to \textbf{Theorem}~\ref{thm:infreliability} and \textbf{Theorem}~\ref{thm:finitereliability}. Our another observation is that, given $\alpha$ fixed, our algorithm can tolerate more agent corruptions if there is a larger difference between $p_1$ and $p_2$, which is also consistent with what we have shown in \textbf{Theorem}~\ref{thm:infreliability} and \textbf{Theorem}~\ref{thm:finitereliability}. Moreover, another interesting observation is that, the larger gap between $p_1$ and $p_2$ implies higher convergence rate, even there may be more corrupted agents in our settings.
    \begin{figure}[htb!]
      \begin{center}
        \parbox{.48\columnwidth}{\center\includegraphics[width=.48\columnwidth]{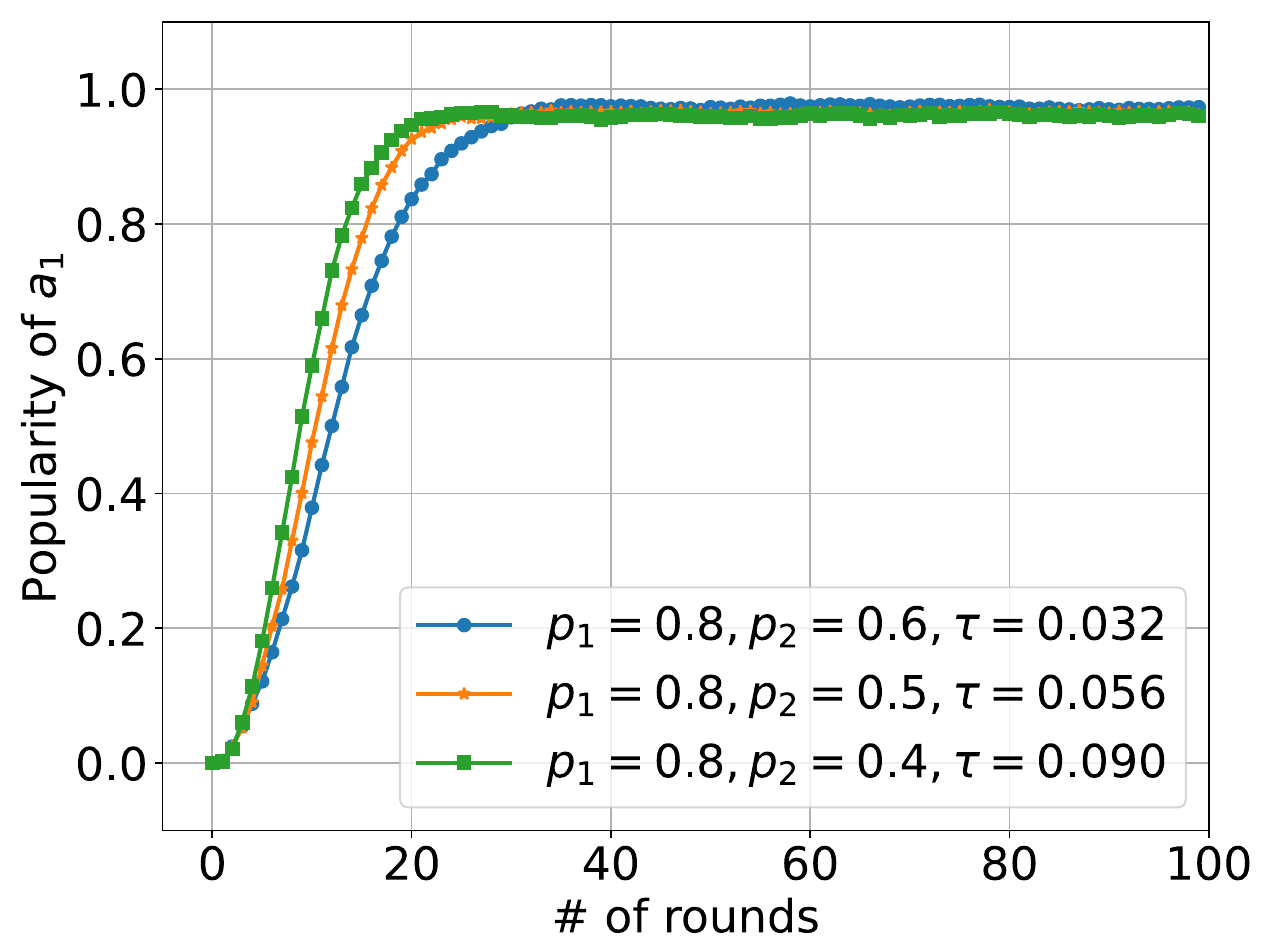}}
        \parbox{.48\columnwidth}{\center\includegraphics[width=.48\columnwidth]{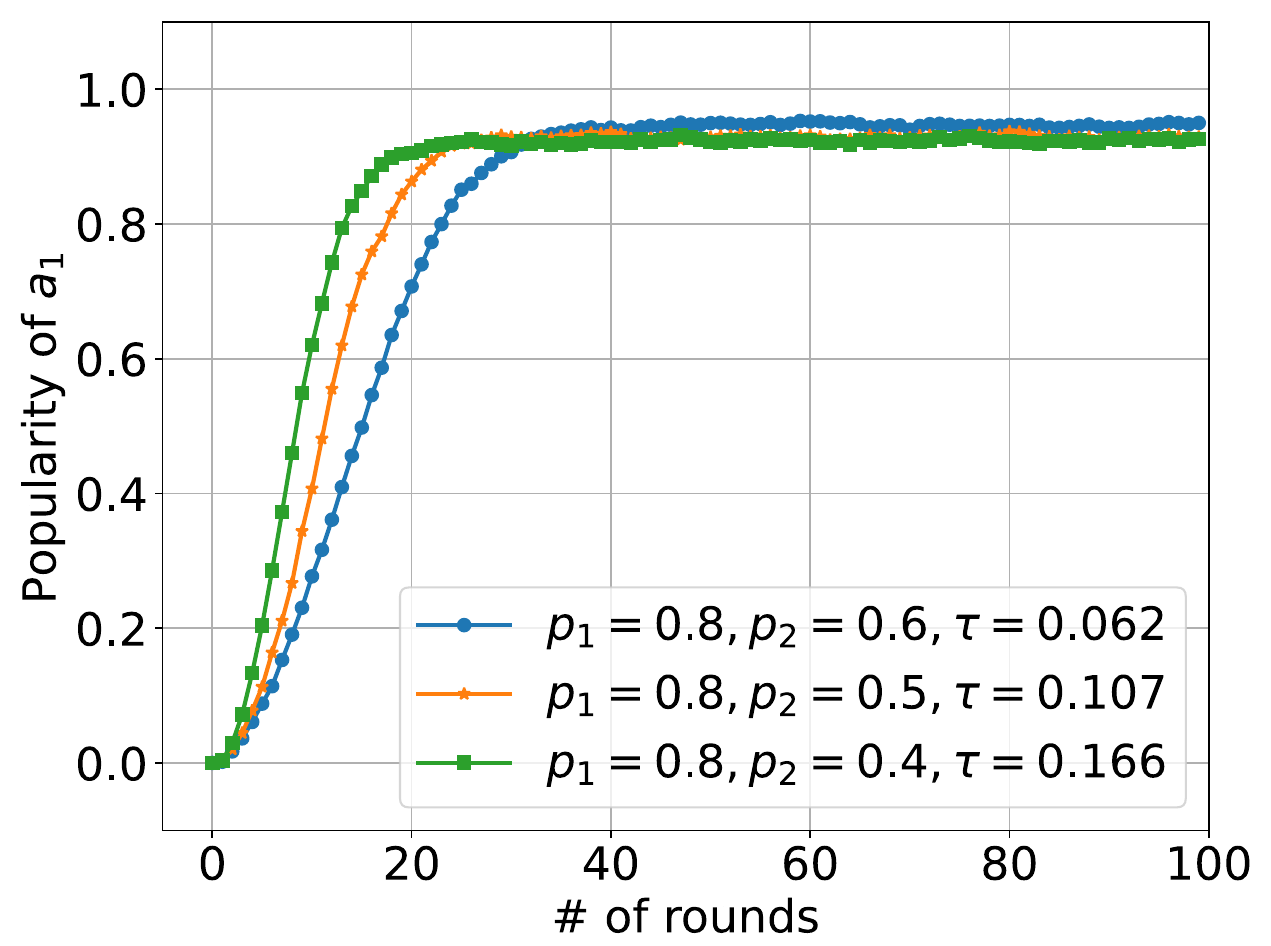}}
        \parbox{.48\columnwidth}{\center\scriptsize(a) $\alpha=0.9$}
        \parbox{.48\columnwidth}{\center\scriptsize(b) $\alpha=0.8$}
        \parbox{.48\columnwidth}{\center\includegraphics[width=.48\columnwidth]{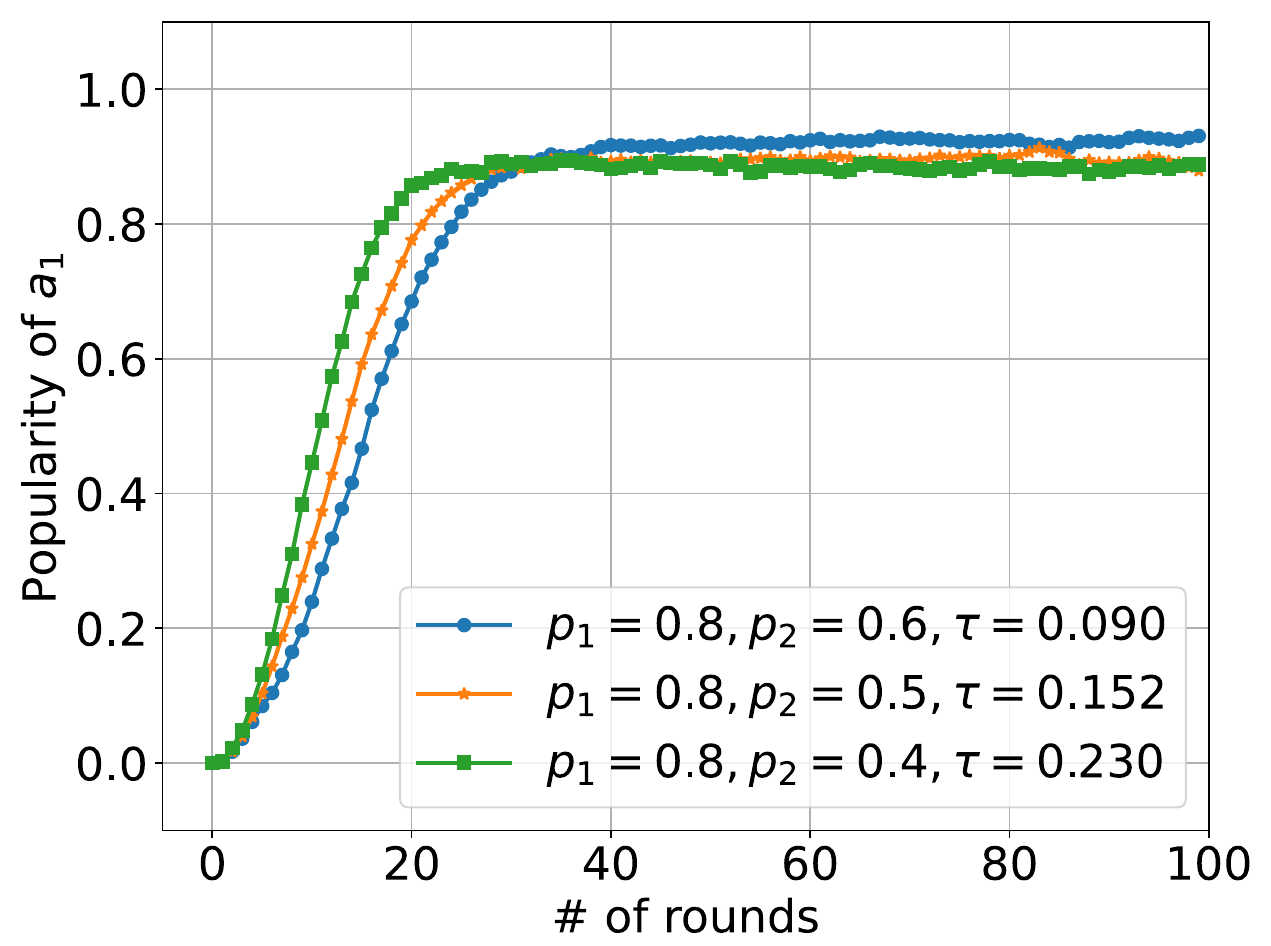}}
        \parbox{.48\columnwidth}{\center\includegraphics[width=.48\columnwidth]{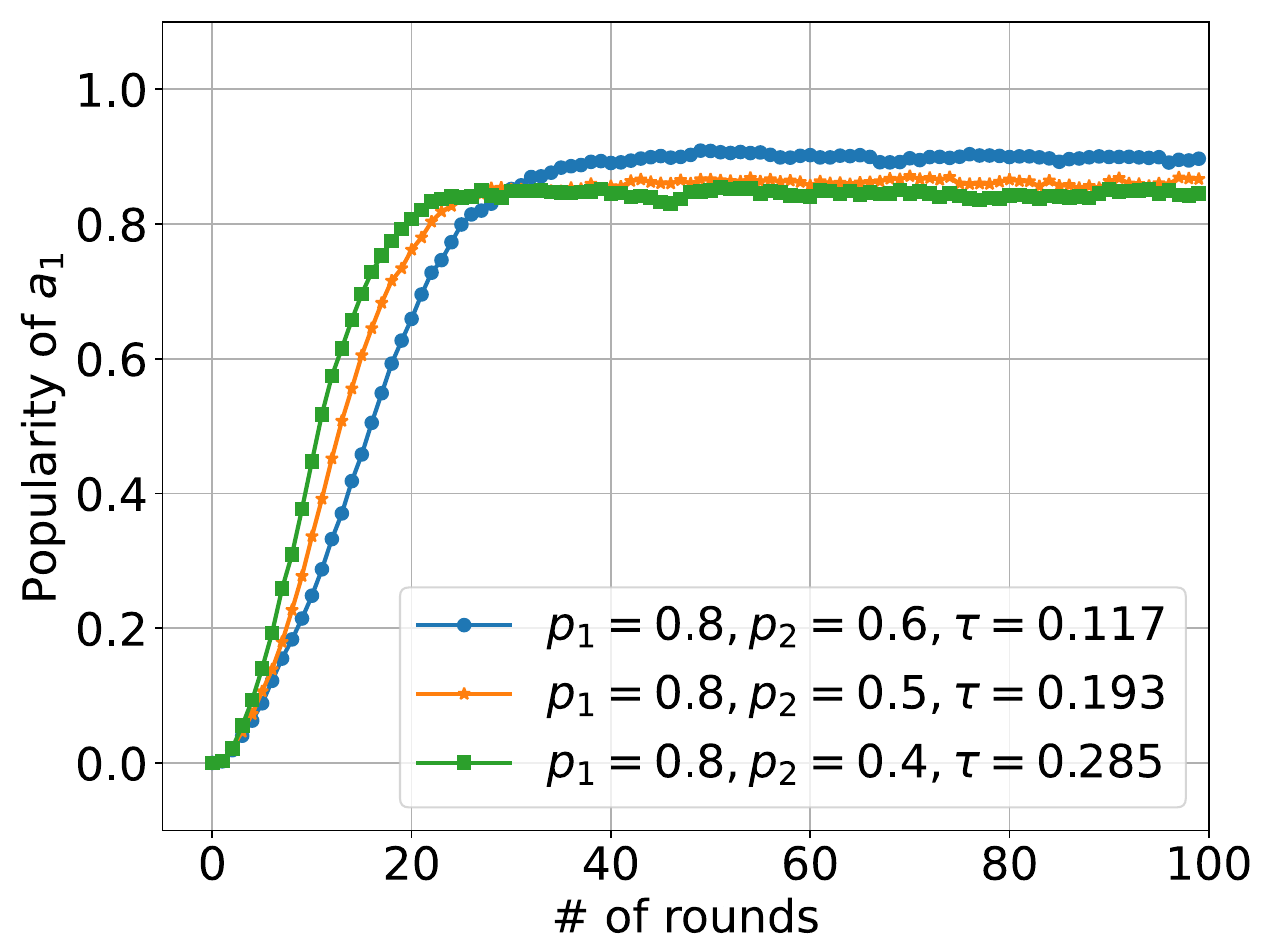}}
        \parbox{.48\columnwidth}{\center\scriptsize(c) $\alpha=0.7$}
        \parbox{.48\columnwidth}{\center\scriptsize(d) $\alpha=0.6$}
        \caption{Reliability under different settings with synthetic dataset.}
        \label{fig:synreliability}
      \end{center}
    \end{figure}

  \subsection{Simulations with Real Data} \label{ssec:real}
    Besides the synthetic dataset, we also evaluate our algorithm with a real dataset. In this paper, We hereby choose Movielens 25M dataset as example, which involves $162,000$ users and $62,000$ movies \cite{mlens}. We select a subset of $3,443$ users and a subset of $707$ movies, such that each of the selected users rated at least 30 of these movies and each of the movies was rated by at least 30 of these users. We extract out the corresponding submatrix and apply the matrix completion method~\cite{HastieMLZ-JMLR15} to fill the missing entries in the extracted submatrix. We then calculate the average of each column and normalize the average to $[0,1]$ by dividing the average by 5. We consider each movie as an arm whose quality can be represented by the normalized score. Likewise, to illustrate the influence of $p_1-p_2$ on the learning process, we fix $p_1=0.9$ (which is the maximum score) and randomly take $500$ samples from the remaining scores such that $p_2\in \{0.7, 0.6, 0.5\}$. We also construct a communication graph where the edges among the users are generated randomly. The results shown in Fig.~\ref{fig:real} is very similar to the ones we obtained with synthetic dataset. In particular, $\mathsf{Succ}(N)$ can be achieved within $40$ rounds even when the gap between $p_1$ and $p_2$ is small (e.g., $p_1-p_2=0.2$), while the temporal overhead can be further reduced with smaller gaps (e.g., around 25 rounds are sufficient when $p_1-p_2=0.4$).
    \begin{figure}[htb!]
      \begin{center}
        \includegraphics[width=.7\columnwidth]{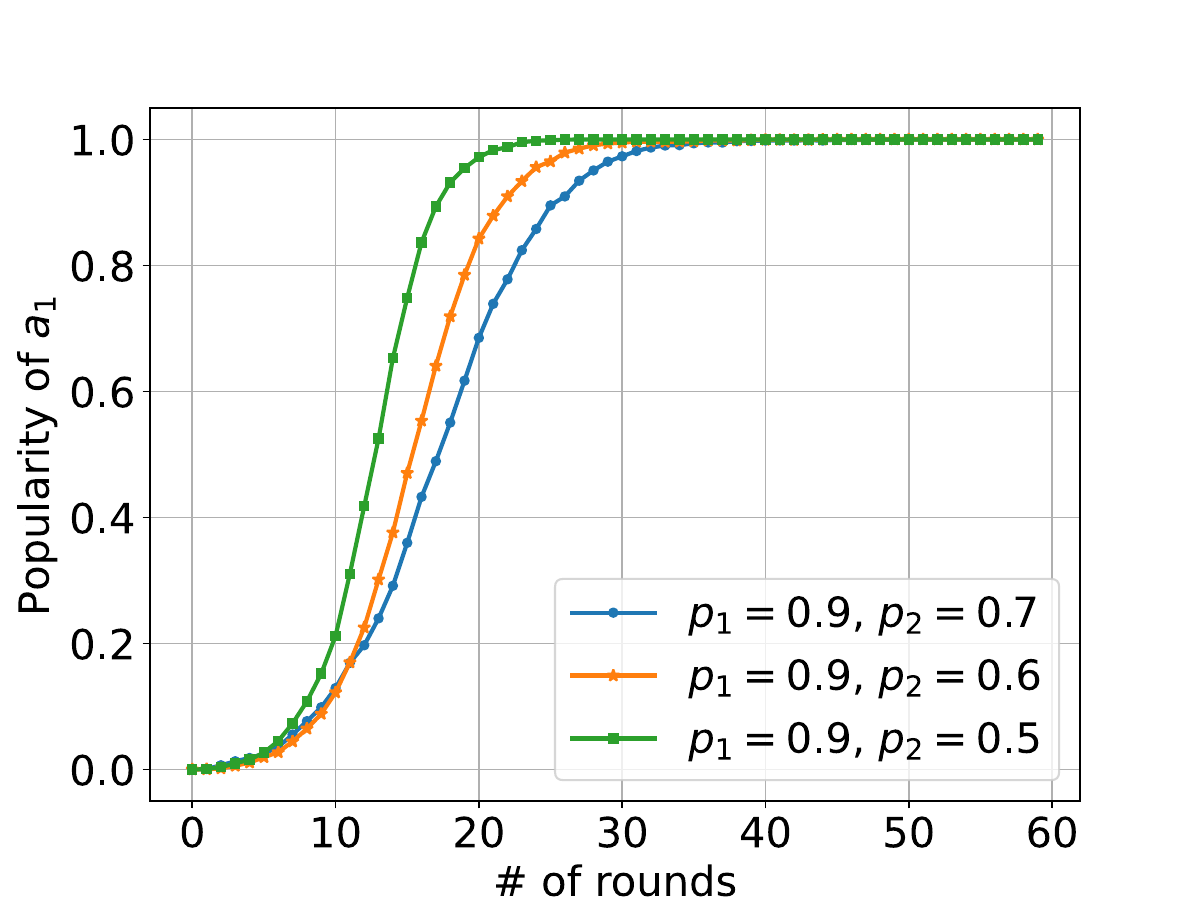}
      \caption{Learnability under different settings with real dataset.}
    \label{fig:real}
    \end{center}
    \end{figure}

    We also evaluate the performance of our algorithm from the perspective of reliability with the real dataset. The experiment results are given in Fig.~\ref{fig:realreliability}. Similar with what we have learnt from the experiment results on the synthetic data, it is demonstrated that we have higher popularity of $a_1$ with larger $\alpha$. Furthermore, with fixed $\alpha$, the convergence rate and the corruption tolerance of our algorithm mainly depend on the difference between $p_1$ and $p_2$. Specifically, our algorithm tolerates more agent corruptions and converges at a higher rate when the second best arm is of much lower quality than the first one. By these observations, our theoretical analysis is further confirmed.
    \begin{figure}[htb!]
      \begin{center}
        \parbox{.48\columnwidth}{\center\includegraphics[width=.48\columnwidth]{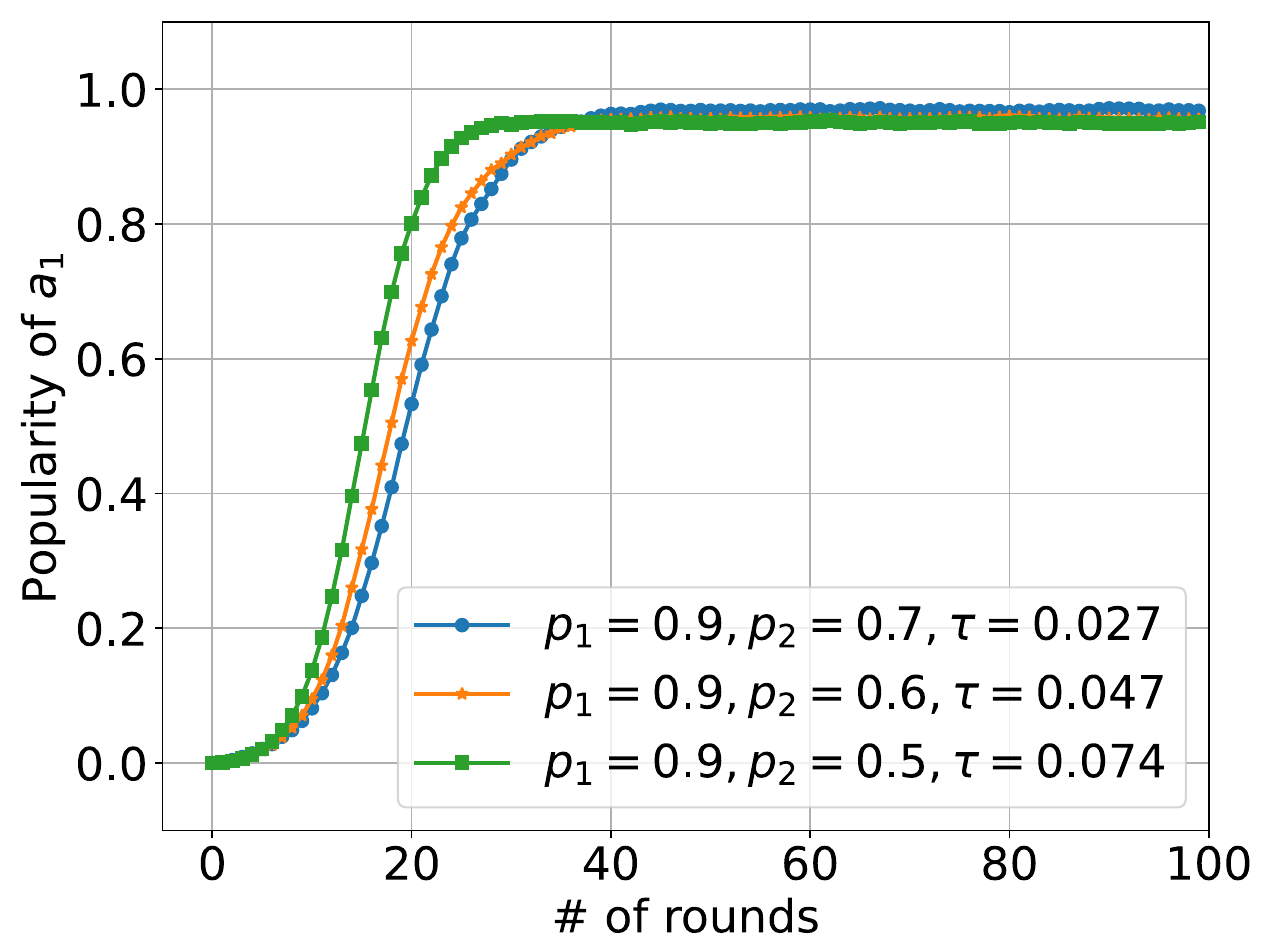}}
        \parbox{.48\columnwidth}{\center\includegraphics[width=.48\columnwidth]{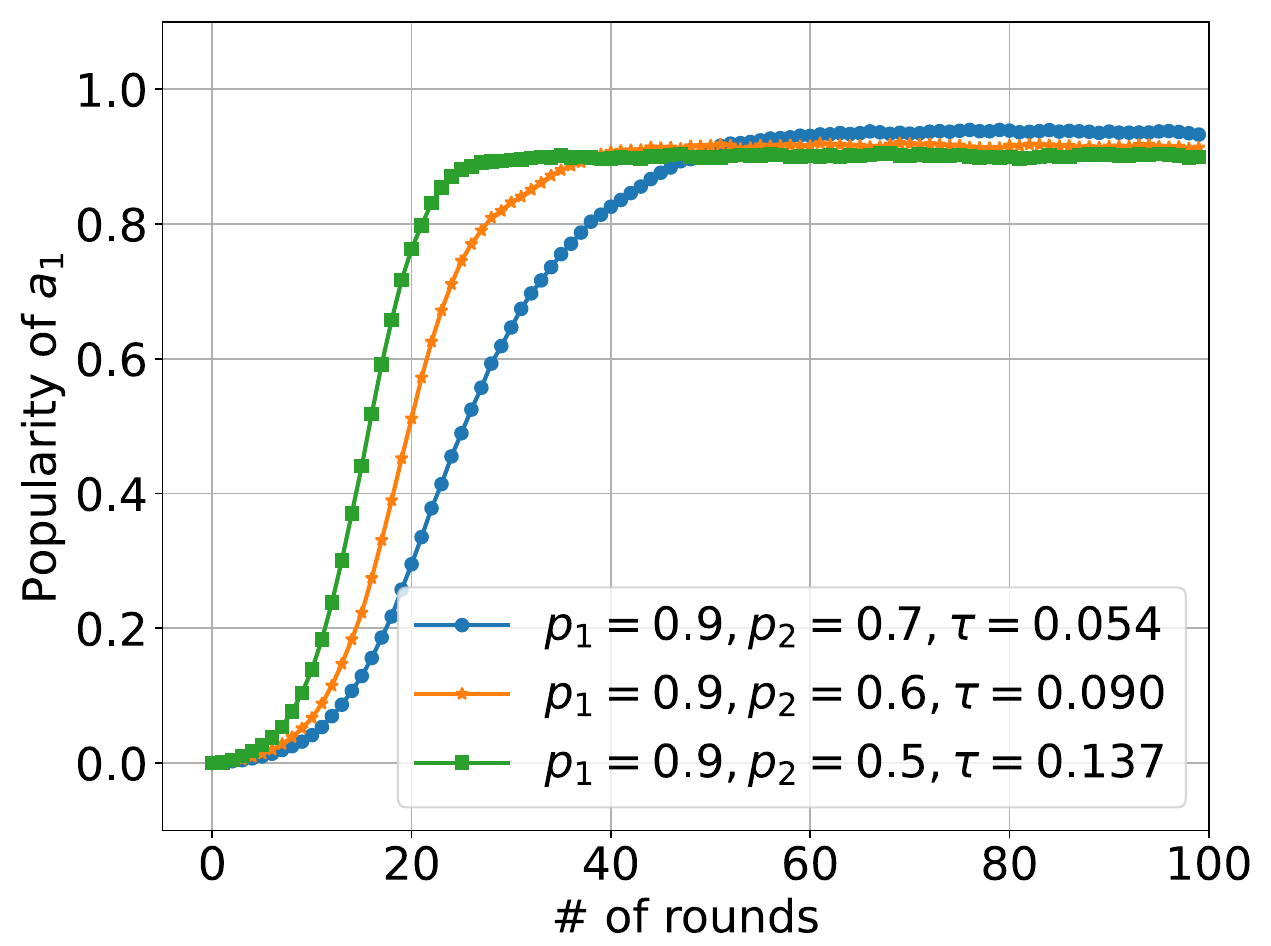}}
        \parbox{.48\columnwidth}{\center\scriptsize(a) $\alpha=0.9$}
        \parbox{.48\columnwidth}{\center\scriptsize(b) $\alpha=0.8$}
        \parbox{.48\columnwidth}{\center\includegraphics[width=.48\columnwidth]{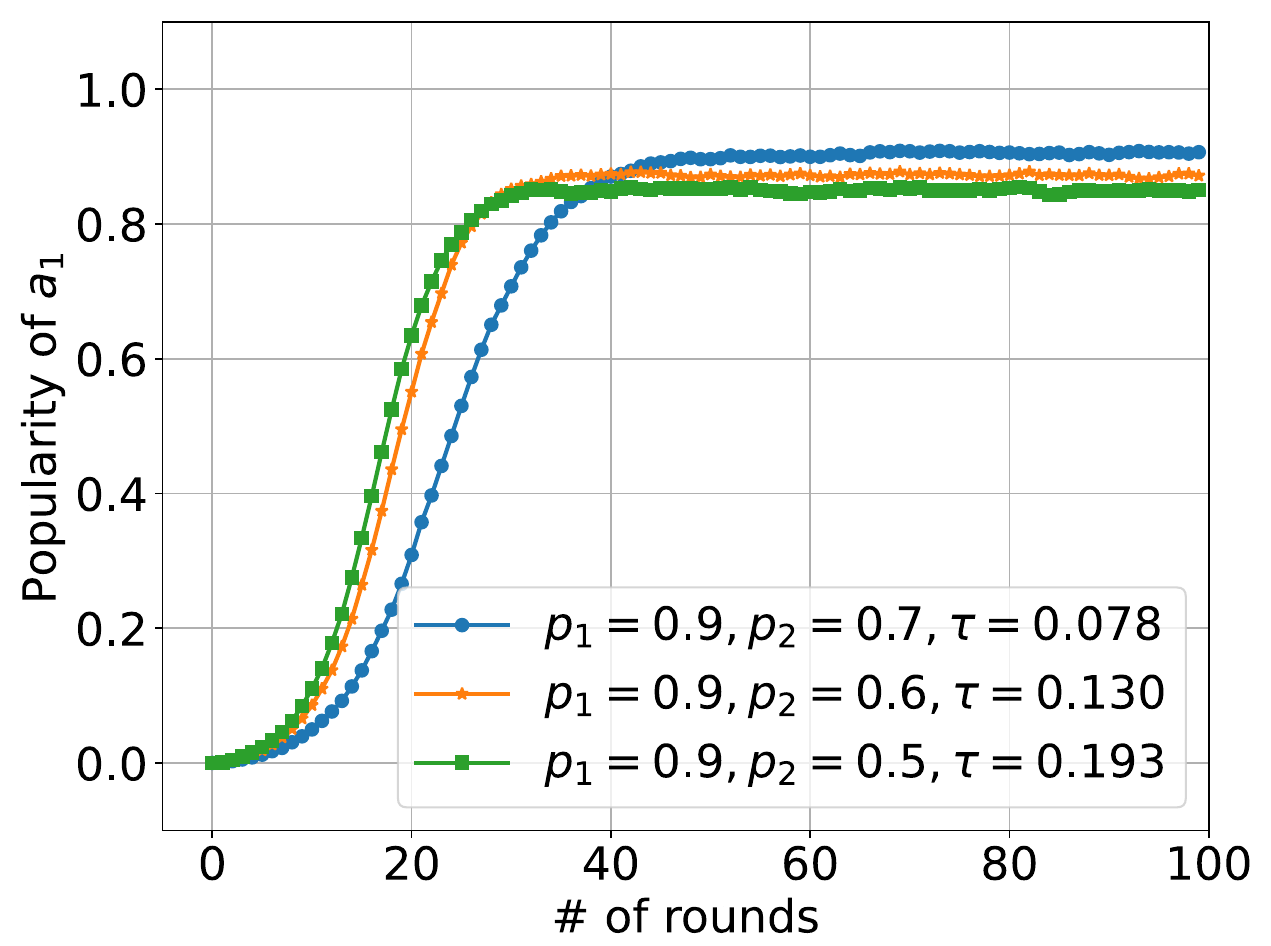}}
        \parbox{.48\columnwidth}{\center\includegraphics[width=.48\columnwidth]{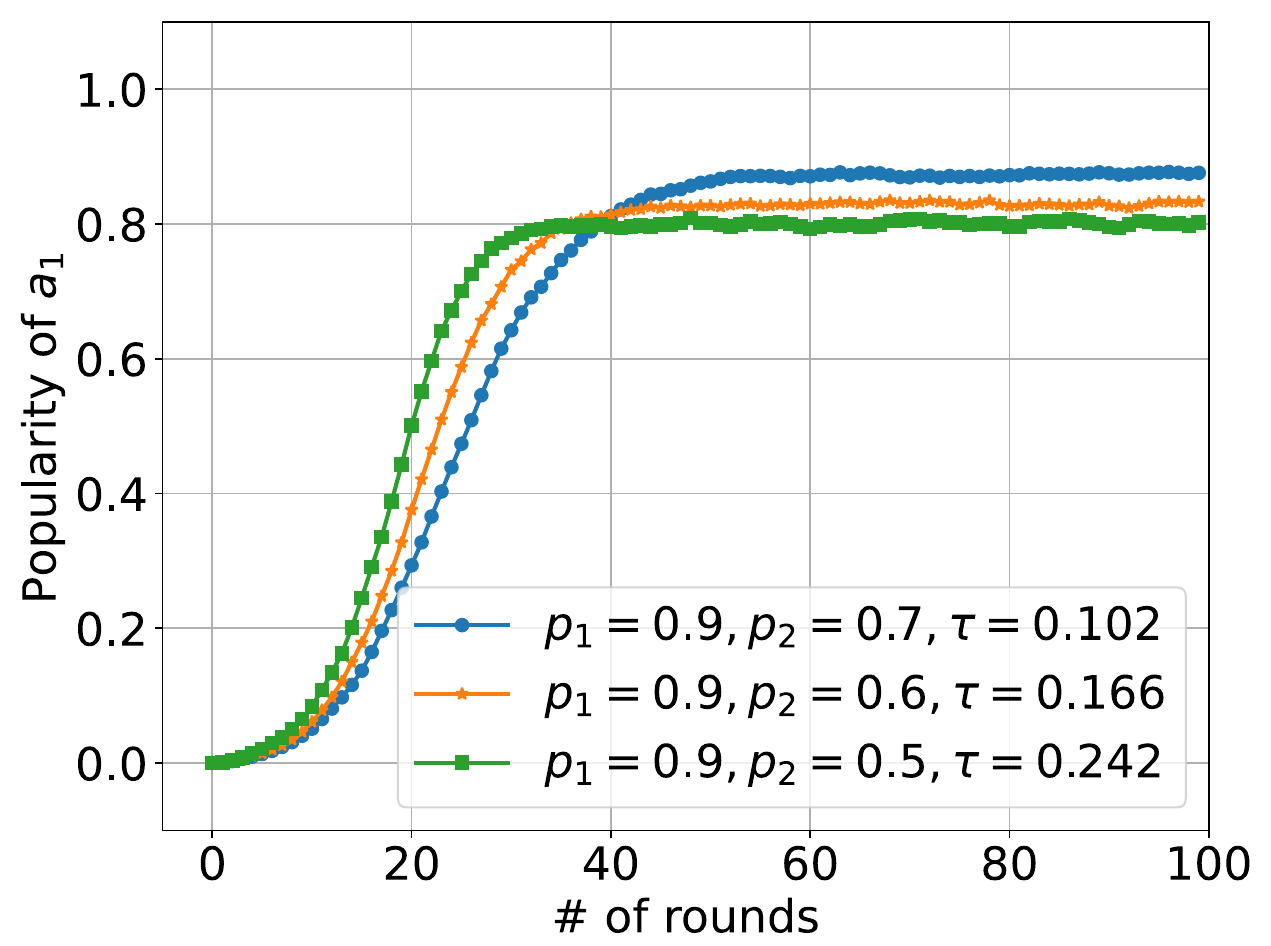}}
        \parbox{.48\columnwidth}{\center\scriptsize(c) $\alpha=0.7$}
        \parbox{.48\columnwidth}{\center\scriptsize(d) $\alpha=0.6$}
        \caption{Reliability under different settings with real dataset.}
        \label{fig:realreliability}
      \end{center}
    \end{figure}

\section{Conclusion and Future work} \label{sec:con}
  In this paper, we have proposed a three-staged collaborative learning algorithm for general multi-agent graphs with constraints on communication bandwidth and memorization. In each round of our algorithm, each agent first disseminates its current adoption (or preference) over the general communication graph through parallel random walks and then chooses one arm to pull according to the suggestions received from its peers. It finally makes an adoption decision according to the observation on the reward yielded by the pulling. According to our theoretical analysis, although the agents have bounded memorizing capabilities such that only the most recent adoptions can be memorized, the learnability of our algorithm can be ensured such that all the agents eventually adopt the best arm with high probability. We also have quantified the reliability of our collaborative learning algorithm in face of agent corruptions. We finally have conducted extensive experiments on both synthetic and real datasets to verify the efficacy of our algorithm.

  In this paper, we assume each agent memorizes only its most recent adoption. Another interesting question is, what if each agent can memorize some of its recent adoptions instead of only the most recent one. For example, an agent may be able to  memorize its adoptions in the most recent $\beta$ rounds. Intuitively, such a relaxed bound on the agents' memorizing capabilities is in favour of the learning process; nevertheless, it is highly non-trivial to quantify the relation between $\beta$ and the performance of the learning algorithm. We are on the way of addressing this challenge.
  %


\bibliographystyle{unsrt}

\bibliography{collalearning}  


\appendix
Recall that $V_i(r)$ denotes the number of tokens received by agent $i$ in the disseminating stage of round $r$. We first illustrate the lower and upper bounds on $V_i(r)$ for any agent $i\in \mathcal N$ in round $r$. Note that, to facilitate our analysis, we suppose $r$ is sufficiently large such that each agent has a non-null adoption, according to \textbf{Lemma}~2.
  \begin{lemma} \label{le:bounds4Vi}
    Assume $N \geq 2$ and $h \geq 64$. In each round $r$, with probability at least $1 - \frac{1}{N^2}$, we have
    \begin{equation} \tag{48}
      \frac{3}{8} h \log N \leq V_i(r) \leq \frac{25}{4} h \log N, ~\forall i\in \mathcal N.
    \end{equation}
  \end{lemma}
  \begin{proof}
    Let $Y_{i,j}(r) \in \{0,1\}$ be an indicator random variable specifying if the $j$-th token reaches agent $i$ in the disseminating stage of round $r$. $V_i(r)$ then can be re-written as $V_i(r) = \sum^{Nh\log N}_{j=1} Y_{i,j}(r)$. Let $\gamma_{i,j} = \mathbb P[Y_{i,j}(r)=1]$ denote the probability that the $j$-th token reaches agent $i$. Hence, we have $\frac{1}{N} - \frac{1}{N^3} \leq \gamma_{i,j} \leq \frac{1}{N} + \frac{1}{N^3}$ according to \textbf{Theorem}~1. By applying the Chernoff-Hoeffding bounds \cite{DubhashiP-book09}, we have
    \begin{equation} \label{eq:rectoken00}
      \mathbb P\left[ \left| V_i(r) - \gamma_i  \right| > \frac{4}{\sqrt{h}}\gamma_i \right] \leq 2 \exp\left(-\frac{16\gamma_i}{3h}\right) \tag{49}
    \end{equation}
    where $\gamma_{i} = \sum^{h N \log N}_{j=1} \gamma_{i,j}$. By applying the union bound across the different agents,
    \begin{align} \label{eq:rectoken01}
      \mathbb P\left[ \left| V_i(r) - \gamma_i  \right| \leq \frac{4}{\sqrt{h}}\gamma_i, ~\forall i\in \mathcal N \right] \geq 1 - 2 N \exp\left(-\frac{16\gamma_i}{3h}\right)  \tag{50}
    \end{align}
    Since $\frac{1}{N} - \frac{1}{N^3} \leq \gamma_{i,j} \leq \frac{1}{N} + \frac{1}{N^3}$, we have
    \begin{align}  \label{eq:rectoken02}
    \begin{cases}
      \gamma_{i} \geq h N \log N \left(\frac{1}{N} - \frac{1}{N^3}\right) = h \log N \left(1 - \frac{1}{N^2}\right) \geq \frac{3h}{4} \log N \\
      \gamma_{i} \leq h N \log N \left(\frac{1}{N} + \frac{1}{N^3}\right) = h \log N \left(1 + \frac{1}{N^2}\right) \leq \frac{5h}{4} \log N   \tag{51}\\
    \end{cases}
    \end{align}
    when $N \geq 2$,  by substituting which into (\ref{eq:rectoken01}), we have
    \begin{align}  \label{eq:rectoken03}
      \frac{3}{4}\left( h - 4\sqrt{h} \right)\log N \leq V_i(r) \leq \frac{5}{4}\left( h + 4\sqrt{h} \right)\log N  \tag{52}
    \end{align}
    hold with probability at least
    \begin{align}
      1 - 2 N \exp \left(-\frac{16\gamma_i}{3h}\right) \geq 1- \frac{2}{N^3} \geq 1 - \frac{1}{N^2}  \tag{53}
    \end{align}
    We complete the proof by substituting $h - 4\sqrt{h} \geq \frac{h}{2}$ and $h + 4\sqrt{h} \leq 5h$ (when $h \geq 64$) into (\ref{eq:rectoken03}).
  \end{proof}
  It is shown in the above lemma that, in each round $r$, every agent receives $\Theta(\log N)$ tokens from its peers. As illustrated in Sec.~4, each agent chooses an arm according to an estimate on its popularity. The question is, based on the received $\Theta(\log N)$ tokens, how accurately can an agent estimate the popularity for any arm? The question is answered in the following \textbf{Lemma}~\ref{le:tail}.
  \begin{lemma} \label{le:tail}
    When $N$ is sufficiently large such that $N \geq 4h \log{N}$, we have the following upper and lower tails for any agent $i$ and any arm $a_k$ in each round $r$
    \begin{align} \label{eq:uptail}
      &\mathbb{P}(Q_{i,k}(r) \geq Q_k(r-1)+\varepsilon \mid V_i(r)=v_i, \mathcal X(r-1))  \nonumber\\
      \leq& \frac{1}{1 - {4h N^{-1} \log{N}}} \exp \left( -\frac{v_i \varepsilon^2}{2+\varepsilon} \right), ~0 \leq \varepsilon \leq 1-Q_k(r-1)  \tag{54}
    \end{align}
    and
    \begin{align} \label{eq:lwtail}
      &\mathbb{P}(Q_{i,k}(r) \leq Q_k(r-1)-\varepsilon \mid V_i(r)=v_i, \mathcal X(r-1))  \nonumber\\
      \leq& \frac{1}{1 - {4h N^{-1} \log{N}}} \exp \left( -\frac{v_i \varepsilon^2}{2+\varepsilon} \right), ~~0 \leq \varepsilon \leq Q_k(r-1)  \tag{55}
    \end{align}
  \end{lemma}
  \begin{proof}
    Recall that there are $M(r) = h N \log N$ tokens disseminated in round $r$, $Q_k(r-1)$ of which are $a_k$-tokens, $V_i(r)$ and $V_{i,k}(r)$ denote the number of tokens and the one of $a_k$-tokens received by agent $i$ in round $r$, respectively, and $Q_{i,k}(r) = V_{i,k}(r) / V_{i}(r)$ is the estimate of agent $i$ on $Q_k(r-1)$. Let $\gamma_{i,j}$ denote the probability that the $j$-th token reaches agent $i$. We also suppose $\mathcal{X}(r-1)$ denote the adoption state in round $r-1$. In the following, to simplify our presentation, we get rid of the indicators of the rounds, i.e, $r$ and $r-1$, in these notations, when doing so does not induce any ambiguity.

    We first look at a simpler cases where the tokens are disseminated uniformly at random such that $\gamma_{i,j} = 1/N$ when $T \rightarrow \infty$. In this case, we have $V_{i,k} \mid V_i,\mathcal X$ obey a hypergeometric distribution $\mathsf{H}(M, Q_k, V_i)$. The expected value of $V_{i,k}  \mid V_i, \mathcal X$ can be calculated as
    \begin{align} \label{eq:hypergeoexp0}
      &\mathbb E[V_{i,k} \mid V_i=v_i, \mathcal X]  \nonumber\\
      =& \sum^{v_i}_{\ell=1} \ell \binom{Q_k M(r)}{\ell} \binom{(1-Q_k) M}{v_i - \ell} {\binom{M}{v_i}}^{-1} = Q_k v_i  \tag{56}
    \end{align}
    and we thus have
    \begin{equation}  \label{eq:hypergeoexp1}
      \mathbb E[Q_{i,k} \mid V_i=v_i, \mathcal X] = Q_k  \tag{57}
    \end{equation}
    According to \cite{Chvatal-DM79,Mulzer-CoRR18}, there is a claim related to the tail of the hypergoemetric distribution which will be useful to our following proof, i.e., for any $\delta \geq 1$,
    \begin{align} \label{eq:claim0}
      {\binom{M}{v_i}}^{-1} \sum^{v_i}_{\ell=0} \binom{Q_kM}{\ell} \binom{(1-Q_k)M}{v_i - \ell} \delta^\ell \leq (1+(\delta-1)Q_k)^{v_i}  \tag{58}
    \end{align}

    In our case, the main difference from the above standard hypergeometric distribution is that the tokens are disseminated in a nearly uniform manner. Specifically, $\gamma_{i,j} \in \left[ \frac{1}{N}-\frac{1}{N^3}, \frac{1}{N}+\frac{1}{N^3} \right]$ for $\forall i,j$. Hence, it is straightforward that
    \begin{align}  \label{eq:probvibd}
      \binom{M}{v_i} \xi_0 \leq \mathbb P(V_i = v_i \mid \mathcal X) \leq \binom{M}{v_i} \xi_1  \tag{59}
    \end{align}
    holds for each agent $i \in \mathcal N$, where
    \begin{align}
      \begin{cases} \label{eq:xi}
      \xi_0 = \left( \frac{1}{N} - \frac{1}{N^3} \right)^{v_i}  \left( 1- \frac{1}{N} - \frac{1}{N^3} \right)^{M-v_i}  \\
      \xi_1 = \left( \frac{1}{N} + \frac{1}{N^3} \right)^{v_i}  \left( 1- \frac{1}{N} + \frac{1}{N^3} \right)^{M-v_i}  \tag{60}
      \end{cases}
    \end{align}
    It is also well known that
    \begin{align}  \label{eq:probjointlwbd}
      \mathbb P(V_{i,k} = v_{i,k}, V_i = v_i  \mid \mathcal X)  \geq \binom{Q_k M}{v_{i,k}} \binom{(1 - Q_k) M}{v_i - v_{i,k}} \xi_0  \tag{61}
    \end{align}
    and
    \begin{align}  \label{eq:probjointupbd}
      \mathbb P(V_{i,k} = v_{i,k}, V_i = v_i  \mid \mathcal X)  \leq \binom{Q_k M}{v_{i,k}} \binom{(1 - Q_k) M}{v_i - v_{i,k}} \xi_1  \tag{62}
    \end{align}
    Therefore, 
    %
    %
    %
    we have
    \begin{align} \label{eq:probvikupbd0}
      & \mathbb P (V_{i,k} = v_{i,k} \mid V_i = v_i, \mathcal X)  \nonumber\\
      =& {\mathbb P ((V_{i,k} = v_{i,k}, V_i = v_i \mid \mathcal X)} \big/ {\mathbb P (V_i = v_i \mid \mathcal X)}  \nonumber\\
      \leq& \binom{Q_k M}{v_{i,k}} \binom{(1-Q_k) M}{v_i - v_{i,k}} {\binom{M}{v_i}}^{-1} \frac{\xi_1}{\xi_0}  \tag{66}
    \end{align}
    by combining (\ref{eq:probvibd})$\sim$(\ref{eq:probjointupbd}). The above bound can be refined by the following inequality. 
    \begin{align} \label{eq:bd4xi10}
      \frac{\xi_1}{\xi_0} =& \left( 1 + \frac{2}{N^2-1} \right)^{v_i} \left( 1 + \frac{2}{N^3-N^2-1} \right)^{M-v_i}  \nonumber\\
      \leq& \left( 1 + \frac{2}{N^2-1} \right)^{h N \log{N}}
      %
      %
      %
      %
      \leq \exp\left(4h N^{-1} \log{N}\right)  \nonumber\\
      \leq& \frac{1}{1 - {4h N^{-1} \log{N}}}  \tag{64}
    \end{align}
    where we have the first inequality since $N^3-N^2-1 \geq N^2-1$ when $N \geq 2$, the second one as $(1+ 1/N)^N \leq e$ holds for any positive integer $N$, and the last one because $\exp\left(4h N^{-1} \log{N}\right)  \leq  {1}/({1 - {4h N^{-1}\log{N}}})$ when $N$ is sufficiently large such that $N \geq 4h \log N$.

    Therefore, for any $\delta \geq 1$ and $v_{i,k} \leq Q_k v_i$, we get
    \begin{align} \label{eq:lwtail00}
      & \mathbb P(V_{i,k} \geq v_{i,k} \mid V_i = v_i, \mathcal X)  \nonumber\\
      =& \sum^{v_i}_{\ell= v_{i,k}} P (V_{i,k} = \ell \mid V_i = v_i, \mathcal X)  \nonumber\\
      %
      %
      %
      %
      \leq& \frac{1}{1 - {4h N^{-1} \log{N}}} {\binom{M}{v_i}}^{-1} \sum^{v_i}_{\ell= v_{i,k}} {\binom{Q_k M}{\ell} \binom{(1- Q_k)M}{v_i - \ell}}  \nonumber\\
      \leq& \frac{{\binom{M}{v_i}}^{-1} \sum^{v_i}_{\ell=0} {\binom{Q_k M}{\ell} \binom{(1- Q_k)M}{v_i - \ell}} \delta^{\ell-v_{i,k}}}{1 - {4h N^{-1} \log{N}}}   \nonumber\\
      \leq& \frac{1}{1 - {4h N^{-1} \log{N}}} \cdot \frac{(Q_k \delta + 1 - Q_k)^{v_i}}{\delta^{v_{i,k}}}  \tag{65}
    \end{align}
    where we have the first inequality by substituting (\ref{eq:bd4xi10}) into (\ref{eq:probvikupbd0}), the second one due to $\delta^{\ell - v_{i,k}} \geq 1$ when $\ell \geq v_{i,k}$ and $\sum^{v_{i,k}-1}_{\ell= 0} {\binom{Q_k M}{\ell} \binom{(1- Q_k)M}{v_i - \ell}} \delta^{\ell - v_{i,k}} \geq 0$ when $\delta \geq 1$, and the last one due to (\ref{eq:claim0}).
    Let $v_{i,k} = (Q_k + \varepsilon) v_i$ (where $0 < \varepsilon < 1-Q_k$) and $\delta = \exp(\lambda)$ (where $\lambda \geq 0$). The above inequality (\ref{eq:lwtail00}) then can be re-written as
    \begin{align}
      &\mathbb P(V_{i,k} \geq (Q_k+\varepsilon)v_i \mid V_i = v_i, \mathcal X)  \nonumber\\
      \leq& \frac{1}{1 - {4h N^{-1} \log{N}}} \left( \frac{Q_k \exp(\lambda) + 1 - Q_k}{\exp(\lambda(Q_k+\varepsilon))} \right)^{v_i}, ~~\forall \lambda>0  \tag{66}
    \end{align}
    To tighten the above bound, we minimize the right hand side of the above inequality with respect to $\lambda$ and obtain
    \begin{equation}
      \exp{(\lambda)} = \frac{(1-Q_k)(Q_k+\varepsilon)}{Q_k(1-Q_k-\varepsilon)}  \tag{67}
    \end{equation}
    by substituting which into (\ref{eq:lwtail00}), we get
    \begin{align} \label{eq:lwtail01}
      &\mathbb P(V_{i,k} \geq (Q_k + \varepsilon) v_i \mid V_i = v_i, \mathcal X)  \nonumber\\
      \leq& \frac{\left( \left( \frac{Q_k}{Q_k + \varepsilon} \right)^{Q_k + \varepsilon} \left( \frac{1 - Q_k}{1 - Q_k - \varepsilon} \right)^{1 - Q_k - \varepsilon} \right)^{v_i}}{1 - {4h N^{-1} \log{N}}}   \nonumber\\
      \leq& \frac{1}{1 - {4h N^{-1} \log{N}}} \exp \left( {-v_i \mathsf{D_{KL}}(Q_k + \varepsilon || Q_k)} \right)  \nonumber\\
      %
      %
      \leq& \frac{1}{1 - {4h N^{-1} \log{N}}} \exp \left( -\frac{v_i \varepsilon^2}{2Q_k + \varepsilon} \right)  \nonumber\\
      \leq& \frac{1}{1 - {4h N^{-1} \log{N}}} \exp \left( -\frac{v_i \varepsilon^2}{2+\varepsilon} \right)  \tag{68}
    \end{align}
    where $\mathsf{D_{KL}}(\cdot || \cdot)$ denotes Kullback-Leibler divergence and we have
    \begin{align*}
      &-\mathsf{D_{KL}}(Q_k + \varepsilon || Q_k)  \nonumber\\
      =& (Q_k + \varepsilon) \ln \frac{Q_k}{Q_k+\varepsilon} + (1-Q_k-\varepsilon) \ln \frac{1-Q_k}{1-Q_k-\varepsilon}  \nonumber\\
      =& -(Q_k + \varepsilon) \ln \left(1+\frac{\varepsilon}{Q_k}\right)  \nonumber\\
      &+ (1-Q_k-\varepsilon) \ln \left(1 + \frac{\varepsilon}{1-Q_k-\varepsilon}\right)  \nonumber\\
      \leq& -(Q_k+\varepsilon)\cdot\frac{\varepsilon/Q_k}{1+\varepsilon/(2Q_k)} + (1-Q_k-\varepsilon)\cdot\frac{\varepsilon}{1-Q_k-\varepsilon}  \nonumber\\
      =& -\frac{\varepsilon^2}{2Q_k+\varepsilon} \leq  -\frac{\varepsilon^2}{2+\varepsilon}
    \end{align*}
    We finally have (\ref{eq:uptail}) proved since $\mathbb P(Q_{i,k} \geq Q_k + \varepsilon \mid V_i = v_i, \mathcal X) = \mathbb P(V_{i,k} \geq (Q_k + \varepsilon) v_i \mid V_i = v_i, \mathcal X)$. 

    The lower tail can be derived according to the above upper one. Specifically, since
    \begin{align}
      & \mathbb P(V_{i,k} \leq (Q_k - \varepsilon)v_i \mid V_i = v_i, \mathcal X)  \nonumber\\
      =& \mathbb P(v_i - V_{i,k} \geq (1  - Q_k + \varepsilon)v_i \mid V_i = v_i, \mathcal X)  \nonumber\\
      =& \mathbb P(V^\neg_{i,k} \geq (Q^\neg_k + \varepsilon)v_i \mid V_i = v_i, \mathcal X)  \tag{69}
    \end{align}
    holds for $0 \leq \varepsilon \leq Q_k$, where $V^\neg_{i,k} = V_i - V_{i,k}$ and $Q^\neg_k = 1 - Q_k$, the lower tail then follows from $\mathsf{D_{KL}}(Q^\neg_k + \varepsilon || Q^\neg_k) = \mathsf{D_{KL}}(Q_k - \varepsilon || Q_k)$.
  \end{proof}

  Now, we are ready to prove \textbf{Lemma}~4. According to \textbf{Lemma}~\ref{le:tail}, in any round $r$, when $V_i(r) \geq \frac{3}{8}h \log N$ and $N$ is sufficiently large such that $N \geq 4 h \log N$ (where $h \geq 64$ is a constant), we have
  \begin{align}
    &\mathbb{P} \left(Q_{i,k}(r) \leq Q_k(r-1)-\varepsilon ~\bigg|~ V_i(r) \geq \frac{3}{8}h \log N, \mathcal X(r-1)\right)  \nonumber\\
    &\leq \frac{N^{-h\varepsilon^2/8}}{1 - {4h N^{-1} \log{N}}}    \tag{70}
  \end{align}
  where $0 \leq \varepsilon \leq Q_k(r-1)$. By applying the union bound across the different agents, we obtain
  \begin{align}
    &\mathbb{P} \left(Q_{i,k}(r) \geq Q_k(r-1)-\varepsilon, \forall i ~\bigg|~ V_i(r) \geq \frac{3}{8}h \log N, \mathcal X(r-1)\right)  \nonumber\\
    &\geq 1 - \frac{N^{1-h\varepsilon^2/8}}{1 - {4h N^{-1} \log{N}}}   \tag{71}
  \end{align}
  for $0 \leq \varepsilon \leq Q_k(r)$. Similarly, for $0 \leq \varepsilon \leq 1-Q_k(r-1)$, we also have
  \begin{align}
    &\mathbb{P} \left(Q_{i,k}(r) \leq Q_k(r-1)+\varepsilon, \forall i ~\bigg|~ V_i(r) \geq \frac{3}{8}h \log N, \mathcal X(r-1)\right)  \nonumber\\
    &\geq 1 - \frac{N^{1-h\varepsilon^2/8}}{1 - {4h N^{-1} \log{N}}}    \tag{72}
  \end{align}
  We finally complete the proof since $\mathbb P(V_i \geq \frac{3}{8}h \log N, \forall i \in \mathcal N \mid \mathcal X(r-1)) \geq 1-1/N^2$ as mentioned in \textbf{Lemma}~\ref{le:bounds4Vi}.

\end{document}